\documentclass[journal]{IEEEtran}
\usepackage[utf8]{inputenc}
\usepackage[pdftex]{graphicx}
\usepackage{epstopdf}
\usepackage{bbm}
\usepackage[utf8]{inputenc} % allow utf-8 input
\usepackage[T1]{fontenc}    % use 8-bit T1 fonts
\usepackage{hyperref}       % hyperlinks
\usepackage{url}            % simple URL typesetting
\usepackage{booktabs}       % professional-quality tables
\usepackage{amsfonts}       % blackboard math symbols
\usepackage{nicefrac}       % compact symbols for 1/2, etc.
\usepackage{microtype}      % microtypography
\newenvironment{proof}[1][Proof]{{\it #1. } }{\ \rule{0.5em}{0.5em}}
\usepackage{graphics, theorem, times, cite}

\usepackage{tikz}
 
\usepackage[english]{babel}
\usepackage{ifpdf}
\usepackage{url}
\usepackage{bm}
\usepackage{hyperref}
\ifpdf
	\usepackage[pdftex]{graphicx}
	\graphicspath{{./figures/}}
\else
	\usepackage[dvips]{graphicx}
	\graphicspath{./figures/}
\fi
\usepackage{color}
\usepackage{pgf, tikz, pgfplots, epstopdf}
\usetikzlibrary{shapes, arrows, automata}
\usepackage{caption} % Extended options for captions.
\usepackage{subcaption} % Similar to caption but for subfigures.
\usepackage{amsmath}
\usepackage{amsfonts, amssymb}
\usepackage{mathrsfs}
\usepackage{upgreek}
\usepackage{algorithm,algpseudocode}
\usepackage{enumerate}
\usepackage{multirow}
\usepackage{rotating}
\usepackage{needspace}

\newcommand{\myparagraph}[1]{\needspace{1\baselineskip}\medskip\noindent {\bf #1.}}

\input{mySymbol.sty}

\usepackage{booktabs}       % professional-quality tables
\usepackage{nicefrac}       % compact symbols for 1/2, etc.
\usepackage{microtype}      % microtypography
\usepackage{xcolor}         % colors

\newtheorem{assumption}{\hspace{0pt}\bf Assumption\hspace{-0.05cm}}
\newtheorem{lemma}{\hspace{0pt}\bf Lemma}
\newtheorem{proposition}{\hspace{0pt}\bf Proposition}

\newtheorem{theorem}{\hspace{0pt}\bf Theorem}
\newtheorem{corollary}{\hspace{0pt}\bf Corollary}

\newtheorem{definition}{\hspace{0pt}\bf Definition}

\newcommand{\QED}{\hfill\ensuremath{\blacksquare}}

\title{Geometric Graph Filters and Neural Networks: \\Limit Properties and Discriminability Trade-offs}

\author{Zhiyang Wang \quad Luana Ruiz \quad Alejandro Ribeiro\thanks{Supported by NSF-Simons MoDL, NSF AI Institutes program and NSF HDR TRipods. ZW, AR are with the Department of Electrical and Systems Engineering, University of Pennsylvania, PA, email: \{zhiyangw, aribeiro\}@seas.upenn.edu. LR is with the Computer Science \& Artificial Intelligence Laboratory at MIT, MA, and is supported by FODSI and METEOR, email: ruizl@mit.edu. Preliminary results presented in \cite{wang2023convolutional} and \cite{wang2023convergence}.}}

\begin{document}

\maketitle

\begin{abstract}
This paper studies the relationship between a graph neural network (GNN) and a manifold neural network (MNN) when the graph is constructed from  a set of points sampled from the manifold, thus encoding geometric information. We consider convolutional MNNs and GNNs where the manifold and the graph convolutions are respectively defined in terms of the Laplace-Beltrami operator and the graph Laplacian. Using the appropriate kernels, we analyze both dense and moderately sparse graphs. We prove non-asymptotic error bounds showing that convolutional filters and neural networks on these graphs converge to convolutional filters and neural networks on the continuous manifold. As a byproduct of this analysis, we observe an important trade-off between the discriminability of graph filters and their ability to approximate the desired behavior of manifold filters. We then discuss how this trade-off is ameliorated in neural networks due to the frequency mixing property of nonlinearities. We further derive a transferability corollary for geometric graphs sampled from the same manifold. We validate our results numerically on a navigation control problem and a point cloud classification task.

\end{abstract}


\section{Introduction}
\label{sec:intro}
%!TEX root = convergence_manifold_main.tex
Geometric data, or data supported in non-Euclidean domains, is the object of much interest in modern information processing. It arises in a number of applications, including protein function prediction \cite{ioannidis2019graph,gligorijevic2021structure}, robot path planning \cite{li2020graph, tolstaya2021multi}, 3D shape analysis \cite{zeng20193d, devanne20143, wang2019dynamic} and wireless resource allocation \cite{he2020resource, wang2022stable}. Graph convolutional filters \cite{gama2020graphs, ortega2018graph} and graph neural networks (GNNs) \cite{scarselli2008graph, gama2019convolutional}, along with manifold convolutional filters \cite{wang2022convolutional} and manifold neural networks (MNNs) \cite{wang2022stability, masci2015geodesic,chakraborty2020manifoldnet}, are the standard tools for invariant information processing on these domains when they are discrete and continuous respectively. The convolution operation is implemented through information diffusion over the geometric structure, thus enabling invariant and stable representations \cite{gama2020stability, keriven2020convergence, zou2020graph, wang2021stabilityjournal} and feature sharing. The cascading neural network architecture interleaves convolutions and nonlinearities, further expanding the model's expressiveness.

Although there is a clear parallel between graphs and manifolds---the former can be seen as discretizations of the latter---, manifolds are infinite-dimensional continuous latent spaces which can only be accessed by discrete point sampling \cite{zeng20193d, monti2017geometric, bronstein2017geometric, eliasof2021pde}. In general, we have access to a set of sampling points from the manifold, and build a graph model to approximate the underlying continuous manifold while attempting to retain the local and global geometric structure \cite{shi2020point, zeng20193d, wang2022stable}. GNNs have been shown to do well at processing information over the manifold both experimentally and theoretically \cite{bronstein2017geometric,wang2022convolutional,kejani2020graph}. Of particular note, conditions that guarantee asymptotic convergence of graph filters and GNNs to manifold filters and MNNs are known \cite{wang2022convolutional}.

Asymptotic convergence is a minimal guarantee that can be enriched with non-asymptotic approximation error bounds. These bounds are unknown and they are the focus of this paper. These non-asymptotic approximation error bounds relating graph filters and GNNs to manifold filters and MNNs are important because they inform the practical design on graphs of information processing architectures that we want to deploy on manifolds. In addition, explicit finite-sample error bounds often reveal details about the convergence regime (e.g., rates of convergence and discriminability trade-offs) that are not revealed by their asymptotic counterparts. For example, the non-asymptotic convergence analysis of GNNs on graphs sampled from a graphon (also referred to as a \textit{transferability} analysis) gives a more precise characterization of the discriminability-convergence tradeoff that arises in these GNNs \cite{ruiz2021transferability}, which is not elucidated by the corresponding asymptotic convergence result \cite{ruiz2021graphonsignal}.

\noindent\textbf{Contributions.} In this paper, we prove and analyze a non-asymptotic approximation error bound for GNNs on graphs sampled from manifold, thus closing the gap between GNNs and MNNs with an explicit numerical relationship. We start by importing the definition of the manifold filter as a convolutional operation where the diffusions are exponentials of the Laplace-Beltrami (LB) operator $\ccalL$ of the manifold $\ccalM \subset \reals^\mathsf{N}$ \cite{wang2022convolutional}. Given a set of discrete sampling points from the manifold, we describe how to construct both dense and relatively sparse geometric graphs that approximate the underlying manifold in both the spatial and the spectral domains. %\red{(L: These sentences have too much detail and are too long. Try to rewrite them more succinctly. You don't need this level of detail here.) Specifically, denoting the geometric graph Laplacian as $\bbL_n^\epsilon$, the point-wise difference bounds for eigenvalues and eigenfunctions between $\bbL_n^\epsilon$ and $\ccalL$ can be given over a limited spectrum. To eliminate the limitation of the spectrum, 

{Next, we import the concept of Frequency Difference Threshold (FDT) filters (Definition \ref{def:alpha-filter}) \cite{wang2022stability} to overcome the challenge of dimensionality associated with the infinite-dimensional spectrum of the LB operator. 
%By giving similar frequency responses to different eigenvalues that are close enough, the FDT filters can trade-off the discriminative power with the approximative power when applying on the geometric graphs.
We show that manifold filters exhibit a trade-off between their discriminability and their ability to be approximated by graph filters,} which can be observed in the approximation error bound of geometric graph filters in Theorems \ref{thm:converge-MF-dense} and \ref{thm:converge-MF-sparse}. 

The same analysis is conducted for GNNs by incorporating nonlinearities (Theorem \ref{thm:converge-MNN}), but in GNNs we hypothesize that the trade-off is alleviated, i.e., that we can recover discriminability, through the addition of these nonlinear operations (Section \ref{subsec:discussion}). In other words, geometric GNNs can be both discriminative and approximative of the underlying MNNs, which we verify empirically through numerical experiments (Section \ref{sec:simu}). Finally, we show how our approximation result can be extended, by a triangle inequality argument, to a transferability result for geometric GNNs transferred across geometric graphs of different sizes sampled from the same underlying manifold. 

\noindent\textbf{Related Works.}  GNNs have been thoroughly studied and discussed in a number of previous works \cite{gama2019convolutional,zhou2020graph,scarselli2008graph}. Similarly, the convergence and transferability of neural networks on graphs have been studied in many works including \cite{ruiz2020graphon, keriven2020convergence,maskey2021transferability,ruiz2021transferability}. These works however see the graphs as samples from a graphon model. In our paper, we focus on the manifold as the limit model for large-scale graphs, which is more realistic than graphons as it can incorporate the underlying geometric information. Moreover, the manifold can also represent the limit of sparse or relatively sparse graphs. Very relevant to this paper, {the spectral convergence of relatively sparse graphs sampled from manifolds has been studied at length in} \cite{calder2019improved}.

Manifold filters and manifold neural networks have been established and discussed in \cite{wang2022convolutional, wang2022stability, masci2015geodesic, monti2017geometric} using different features and metrics. {The integral of the heat diffusion process is employed in \cite{wang2022convolutional, wang2022stability} where it is shown that manifold convolutions are consistent with and can recover both the graph convolution and the time convolution. In \cite{masci2015geodesic}, a local geodesic system capturing local anisotropic structures is used to define the manifold convolution. Taking a different approach, \cite{monti2017geometric} defines a mixture Gaussian kernel over the spatial domain using patch operator that is constructed locally. } %\red{(L: You have to discuss at least some of the works in the previous sentence (at least those more closely related to yours) in more detail.)} 

Among all these methods, the definition leveraging the heat diffusion process is the only one to build the connection between graph neural networks and manifold neural networks when the graphs are sampled from the manifold \cite{wang2022convolutional}. 
%In \cite{wang2022convolutional}, it was also shown that MNNs can recover GNNs by discretization over the space and time domains, and that GNNs converge to MNNs as the number of sampling points grows. 
However, this paper does not provide an explicit convergence rate or approximation error bound between GNNs and MNNs. 
In \cite{chew2022convergence}, a convergence rate is given for GNNs approximating MNNs, but the result is restricted to input signals with a limited bandwidth. Moreover, the graphs constructed in \cite{wang2022convolutional} and \cite{chew2022convergence} still need to be dense with a well-defined Gaussian kernel.
%and there is a bandlimited input signal assumption which makes the analysis less general. 
In contrast, in this paper we extend the analysis of the convergence of GNNs by proving a non-asymptotic approximation error bound for geometric GNNs sampled from the MNNs where the geometric graphs are either dense or relatively sparse; and we eliminate the bandlimited signal assumption by importing frequency-dependent filters. 
\myparagraph{Organization} Section \ref{sec:preliminary} introduces preliminary definitions of graph signal processing and graph neural networks along with manifold signal processing and manifold neural networks. Section \ref{sec:geom_graphs} introduces the construction of the geometric graphs from the sampling points of the manifold and the convergence of the geometric graph Laplacians. We show the spectrum of the geometric graphs -- both dense and relatively sparse -- can approximate the spectrum of the LB operator of the underlying manifold with a bounded error. We then go on to study the approximation of the filtering on the geometric graphs which is further extended to the approximation of geometric GNNs to MNNs in Section \ref{sec:converge_gnn}. Section \ref{subsec:discussion} discusses the implications of the results derived in Section \ref{sec:converge_gnn}. Section \ref{sec:simu} illustrates the results of Sections \ref{sec:converge_gnn} with numerical examples. Section \ref{sec:conclusion} concludes the paper. Proofs are deferred to the appendices and supplementary materials.

\section{Preliminary Definitions}
\label{sec:preliminary}
We start by reviewing the basic architecture of graph neural networks and manifold neural networks. 

\subsection{Graph Signal Processing and Graph Neural Networks}

Let $\bbG = (\ccalV, \ccalE, \ccalW)$ be an undirected graph with node set $\ccalV, |\ccalV|=n$, and edge set  $\ccalE \subseteq \ccalV\times \ccalV$. The edges in $\ccalE$ may be weighted, in which case the edge weights are assigned by a function $\ccalW: \ccalE \rightarrow \reals$. In this paper, we are interested in graph signals $\bbx \in \reals^n$ supported on the nodes. For $1 \leq i \leq n$, $[\bbx]_i$ represents the value of the signal at node $i$. 

\noindent \textbf{Graph shift.} Operating on graph signals $\bbx$, we define the graph shift operator (GSO) $\bbS \in \reals^{n\times n}$. The GSO is any matrix satisfying $[\bbS]_{ij} \neq 0$ if and only if $(i,j)\in \ccalE$ or $i=j$, e.g., the adjacency matrix $\bbA$, $[\bbA]_{ij} = \ccalW(i,j)$, or the graph Laplacian $\bbL = \mbox{diag}(\bbA\boldsymbol{1})-\bbA$ \cite{ortega2018graph, shuman2013emerging}. The GSO is so called because, at each node $i$, it has the effect of shifting or diffusing the signal values at $i$'s neighbors to $i$, where  signal values are aggregated. Explicitly, $[\bbS\bbx]_i=\sum_{j,(i,j)\in\ccalE}[\bbS]_{ij}[\bbx]_j$.
In the case of an undirected graph, the GSO is symmetric and can therefore be decomposed as $\bbS=\bbV \bm{\Lambda} \bbV^H$. The eigenvalues in the diagonal matrix $\bm\Lambda$ are seen as graph frequencies, and the eigenvectors in $\bbV$ are the corresponding graph oscillation modes. 

\noindent \textbf{Graph convolution.} A graph convolution is defined based on the graph diffusion process with $K_t-1$ consecutive shifts. More formally, the graph convolutional filter \cite{gama2019convolutional,ortega2018graph,sandryhaila2013discrete} with coefficients $\{h_k\}_{k=0}^{K_t-1}$ is given by 
\begin{equation}
    \label{eqn:graph_convolution}
\bbh_\bbG(\bbS) \bbx = \sum_{k=0}^{K_t-1} h_k \bbS^k \bbx .
\end{equation}
Plugging the spectral decomposition of $\bbS$ into \eqref{eqn:graph_convolution}, we see that in the spectral domain this filter can be represented as
\begin{equation}
    \label{eqn:graph_convolution_spectral}
    \bbV^H \bbh_\bbG(\bbS) \bbx =  \sum_{k=1}^{K_t-1} h_k \bm\Lambda^k \bbV^H \bbx = h(\bm\Lambda)\bbV^H \bbx.
\end{equation}
Hence, the graph frequency response of the graph convolution is given by $h(\lambda)= \sum_{k=0}^{K_t-1} h_k \lambda^k$, which only depends on the weights $h_k$ and on the eigenvalues of $\bbS$.

\noindent \textbf{Graph neural network.} A graph neural network (GNN) is built by cascading layers that each consists of a bank of filters followed by a point-wise nonlinearity $\sigma:\reals\to\reals$, $[\sigma(\bbx)]_i = \sigma([\bbx]_i)$. The $l$th layer of a GNN produces $F_l$ signals $\bbx_l^p$, each called a feature, given by
\begin{equation} \label{eqn:gnn}
    \bbx_l^p = \sigma\left(\sum_{q=1}^{F_{l-1}} \bbh_\bbG^{lpq}(\bbS) \bbx^q_{l-1} \right),
\end{equation}
for $1 \leq  p \leq F_l$. At each layer $l=1,2\hdots, L$, the number of input and output features are $F_{l-1}$ and $F_l$ respectively. The filter $\bbh_\bbG^{lpq}(\bbS)$ is as in \eqref{eqn:graph_convolution} and maps the $q$-th feature of layer $l-1$ to the $p$-th feature of layer $l$. For simplicity, we write the GNN consisting of $L$ layers like \eqref{eqn:gnn} as the map $\bm\Phi_\bbG(\bbH, \bbS, \bbx)$, where $\bbH$ denotes a set of the graph filter coefficients at all layers.

\subsection{Manifold Signal Processing and Manifold Neural Networks}

{Let $\ccalM$ be a $d$-dimensional embedded manifold in $\reals^\mathsf{N}$. 
For simplicity, in this paper whenever we mention the {manifold} $\ccalM$, we assume that it is a compact, smooth and differentiable $d$-dimensional submanifold embedded in $\reals^\mathsf{N}$.}  Atop $\ccalM$, signals are defined as scalar functions $f:\ccalM\rightarrow \reals$ called manifold signals \cite{wang2021stabilityjournal}. The inner product of signals $f, g\in L^2(\ccalM)$ is defined as 
\begin{equation}\label{eqn:innerproduct}
    \langle f,g \rangle_{\ccalM}=\int_\ccalM f(x)g(x) \text{d}\mu(x), 
\end{equation}
where $\text{d}\mu(x)$ is the volume element with respect to the measure $\mu$ over $\ccalM$. Similarly, the norm of the manifold signal $f$ is
\begin{equation}\label{eqn:manifold_norm}
    \|f\|^2_{\ccalM}={\langle f,f \rangle_{\ccalM}}.
\end{equation}

\noindent \textbf{Manifold shift.} The manifold $\ccalM$ is locally Euclidean, and the local homeomorphic Euclidean space around each point $x\in\ccalM$ is defined as the tangent space $T_x\ccalM$. The disjoint union of all tangent spaces over $\ccalM$ is called the tangent bundle and denoted $T\ccalM$. The \emph{intrinsic gradient} $\nabla: L^2(\ccalM)\rightarrow L^2(T\ccalM)$ is the differentiation operator and maps scalar functions to tangent vector functions over $\ccalM$ \cite{bronstein2017geometric, boumal2023introduction}. The adjoint of $\nabla$ is the \emph{intrinsic divergence}, which is defined as $\text{div}: L^2(T\ccalM)\rightarrow L^2(\ccalM)$. The Laplace-Beltrami (LB) operator $\ccalL: L^2(\ccalM) \to L^2(\ccalM)$ is defined as the intrinsic divergence of the intrinsic gradient \cite{rosenberg1997laplacian}. Formally, the operation is written as
\begin{equation}\label{eqn:Laplacian}
    \ccalL f=-\text{div}\circ \nabla f.
\end{equation}
This operator measures the difference between the signal value at a point and the average signal value in the point's neighborhood \cite{bronstein2017geometric}. This is akin to how the graph Laplacian matrix can be used to compute the total variation of a graph signal \cite{chung1997spectral}.

On manifolds, the shift operation is defined based on the LB operator and on the solution to the heat equation (see \cite{wang2021stabilityjournal} for a detailed exposition). Explicitly, for a manifold signal $f$ the manifold shift is written as $e^{-\text{d}t\ccalL}f$. Since the LB operator is self-adjoint and positive-semidefinite and the manifold $\ccalM$ is compact, $\ccalL$ has real, positive and discrete eigenvalues $\{\lambda_i\}_{i=1}^\infty$, which can be written as 
\begin{equation}\label{eqn:laplace-decomp}
\ccalL \bm\phi_i =\lambda_i \bm\phi_i,
\end{equation}
where $\bm\phi_i$ is the eigenfunction associated with eigenvalue $\lambda_i$. The eigenvalues, or manifold frequencies, are ordered in increasing order as $0<\lambda_1\leq \lambda_2\leq \lambda_3\leq \hdots$, and the eigenfunctions, or manifold oscillation modes, are orthonormal and form an eigenbasis of $L^2(\ccalM)$ in the intrinsic sense. The $\phi_i$ are also eigenfunctions of the shift operator $e^{-\text{d}t\ccalL}$, with corresponding eigenvalues $e^{-\text{d}t\lambda_i}$.

\noindent \textbf{Manifold convolution.} Integrating $e^{-\text{d}t\ccalL}f$ over $[0,\infty)$ yields the infinite-horizon diffusion process
\begin{equation}
    g(x) = \int_0^\infty e^{-t\ccalL} f(x) \text{d}t \text{.}
\end{equation}
 Let $\tdh:\reals^+ \to \reals$ denote the filter impulse response. Based on this diffusion process,a manifold filter can be defined via the manifold convolution \cite{wang2021stabilityjournal},
\begin{align} \label{eqn:convolution-conti}
   g(x) = \bbh f(x) := \int_0^\infty \tdh(t)e^{-t\ccalL}f(x)\text{d}t =  \bbh(\ccalL)f(x) \text{.}
\end{align}
% Similar to the convolution on time or image signals in Euclidean domain as well as the graph convolutions in \eqref{eqn:graph_convolution}, manifold convolution operates on manifold signal $f(x)$ by scaling the heat diffusion condition with $\hat{h}(t)$ and aggregating the scaled signal along the continuous time horizon from $t=0$ to $t=\infty$. By analogy, the exponential term $e^{-t\ccalL}$  can be seen as a shift similar to a time delay in a linear time-invariant (LTI) filter \cite{oppenheim1997signals}. Also, as the graph shift operator in a linear shift-invariant (LSI) graph filter \cite{gama2020graphs}. 
Note that the map $\bbh:=\bbh(\ccalL)$ is parametric on the LB operator, and is a spatial map acting directly on $x\in \ccalM$. 

If we write $[\hat{f}]_i=\langle f, \bm\phi_i\rangle_{L^2(\ccalM)}=\int_{\ccalM} f(x)\bm\phi_i(x) \text{d}\mu(x)$, the manifold convolution can be represented in the manifold frequency domain as
\begin{align}
    [\hat{g}]_i = \int_0^\infty \tdh(t) e^{-t\lambda_i}  \text{d} t [\hat{f}]_i\text{.}
\end{align}
Hence, the frequency response of the filter $\bbh(\ccalL)$ is given by $\hat{h}(\lambda)=\int_0^\infty \tdh(t) e^{- t \lambda  }\text{d}t$, which only depends on the impulse response $\tilde{h}$ and the LB eigenvalues $\lambda = \lambda_i$. 
%As the function $f$ can also be represented on the eigenbasis $\{\bm\phi_i\}_{i=1}^\infty$ as $f=\sum_{i=1}^\infty [\hat{f}]_i\bm\phi_i$, the spectral representation of manifold convolution is written as 
Further summing over all $i$ and projecting back onto the spatial domain, we can alternatively represent $\bbh(\ccalL)$ as
\begin{equation}\label{eqn:spectral-filter}
    g =\bbh(\ccalL)f= \sum_{i=1}^{\infty} \hat{h}(\lambda_i)[\hat{f}]_i \bm\phi_i .
\end{equation}

\noindent \textbf{Manifold neural network.} Similarly to the GNN, we can define the manifold neural network (MNN) as a cascade of layers $l=1 ,2 \hdots, L$ each of which consists of a bank of manifold filters and a pointwise nonlinearity $\sigma:\reals \to \reals$. Layer $l$ is explicitly written as
\begin{equation}\label{eqn:mnn}
f_l^p(x) = \sigma\left( \sum_{q=1}^{F_{l-1}} \bbh_l^{pq}(\ccalL) f_{l-1}^q(x)\right),
\end{equation}
where each signal $f_l^p$, $1 \leq p \leq F_l$ is a different feature. Each layer $l$ maps $F_{l-1}$ input features to $F_l$ output features. 
%The output of the MNN is given by the outputs of the $L$-th layer, i.e., $f_L^p$ for $1 \leq p \leq F_L$. 
For a more succinct representation, in the following the MNN will be denoted $\bbPhi(\bbH,\ccalL,f)$, where $\bbH$ is a function set gathering the impulse responses of the manifold filters $\bbh_l^{pq}$ for all features and all layers.
%with the filter functions and LB operator as parametrizations. 

\section{Geometric Graphs and Laplacian Convergence}
\label{sec:geom_graphs}

%\subsection{Graph construction}
{Given discrete set of points sampled from the manifold $\ccalM$, we can construct a discrete approximation of $\ccalM$ by connecting the sample points by means of a graph. The nodes of the graph are the sample points, and the edges (more specifically, their weights) are defined as some function of the Euclidean distance between each node pair to encode geometric information from the manifold. We henceforth refer to graphs carrying topological manifold information as \emph{geometric graphs}, which is slightly different from the graph theory definition that focuses on the geometric properties of the edges of the graph \cite{pach2013beginnings}.} 
%In the following context, we refer to \emph{geometric graphs} as the graphs constructed from sampled points involving geometric manifold information.

Let $X$ be a set of $n$ points $\smash{\{x_1, x_2, \hdots , x_n\}}$ sampled i.i.d. from the manifold $\ccalM$ according to measure $\mu$. The discrete empirical measure associated with these points is defined as $p_n=\frac{1}{n}\sum_{i=1}^n \delta_{x_i}$, where $\delta_{x_i}$ represents the Dirac measure at $x_i$. For signals $u,v \in L^2(\ccalM)$, the inner product associated with measure $p_n$ is defined as 
 \begin{equation}
     \langle u, v\rangle_{L^2(\bbG_n)}=\int u(x)v(x)\text{d}p_n=\frac{1}{n}\sum_{i=1}^n u(x_i)v(x_i)
 \end{equation}
 and so the norm in $L^2(\bbG_n)$ is 
$$\|u\|^2_{L^2(\bbG_n)} = \langle u, u \rangle_{L^2(\bbG_n)}.$$ For signals $\bbu,\bbv \in L^2(\bbG_n)$, the inner product is $$\langle \bbu,\bbv \rangle_{L^2(\bbG_n)} = \frac{1}{n}\sum_{i=1}^n [\bbu]_i[\bbv]_i$$
and the norm is $\smash{\|\bbu\|_{L_2(\bbG_n)} = \langle\bbu,\bbu\rangle_{L^2(\bbG_n)}}$.

We construct an undirected \textit{geometric graph} $\bbG_n$ from the sampled points $X$ by seeing each point as a vertex. Every pair of vertices is connected by edges with weight values determined by a function $K_\epsilon$ of their Euclidean distance. Explicitly, the weight of edge $(i,j)$, denoted $w_{ij}$, is given by
\begin{equation}\label{eqn:weight}
w_{ij}=K_\epsilon \left(\frac{\|x_i-x_j\|^2}{\epsilon}\right),
\end{equation}
where $\|x_i-x_j\|$ is the Euclidean distance between $x_i$ and $x_j$. The geometric adjacency matrix $\bbA_n^\epsilon$ and the geometric graph Laplacian $\bbL_n^\epsilon$ can therefore be expressed as $[\bbA_n^\epsilon]_{ij}=w_{ij}$ for $1 \leq i,j\leq n$ and $\bbL^\epsilon_n = \mbox{diag}(\bbA_n^\epsilon \boldsymbol{1})-\bbA_n^\epsilon$ \cite{merris1995survey}.

Since it is constructed from points $x_i \in \ccalM$, the geometric graph can be seen as a discrete approximation, or discretization, of the manifold. To analogously obtain the discretization of a manifold signal $f\in L^2(\ccalM)$ on this graph, as well as the reconstruction from the graph signal back to the manifold signal, we define a uniform sampling operator $\bbP_n: L^2(\ccalM)\rightarrow L^2(\bbG_n)$ and an interpolation operator $\bbI_n:L^2(\bbG_n)\rightarrow L^2(\ccalM)$. The application of the operator $\bbP_n$ to $f$ yields the \textit{geometric graph signal}
\begin{equation}
\label{eqn:sampling}
    \bbf = \bbP_n f\text{ with }\bbf(x_i) = f(x_i), \quad  x_i \in X \text{.}
\end{equation}
I.e., $\bbf$ is a signal on the graph $\bbG_n$ sharing the values of the manifold signal $f$ at the sampled points $X$. 

We put certain restrictions on the sampling and interpolation operators as follows \cite{levie2019transferability}.
\begin{definition}
We call the sampling and interpolation operators $\bbP_n$ and $\bbI_n$ \emph{asymptotically reconstructive} if for any manifold signal $f\in L^2(\ccalM)$, it holds
\begin{equation}
    \lim_{n\rightarrow \infty} \bbI_n \bbP_n f = f.
\end{equation}
Moreover, the sampling and interpolation operators $\bbP_n$ and $\bbI_n$ are bounded if there exists a constant $D$ such that
\begin{equation}
   \limsup_{n\in\mathbb{N}}{\|\bbP_n\|\ }\leq D, \quad  \limsup_{n\in\mathbb{N}}{ \|\bbI_n\|\ }\leq D.
\end{equation}
\end{definition}

Seeing the geometric graph Laplacian $\bbL^\epsilon_n$ as an operator acting on $\bbf: X\to \reals$, we can write the diffusion operation at each point $x_i$ explicitly as
\begin{equation}
\label{eqn:graph_laplacian}
  \bbL^\epsilon_n \bbf(x_i) = \sum_{j=1}^n K_\epsilon \left(\frac{\|x_i-x_j\|^2}{\epsilon}\right)\left( \bbf(x_i) -\bbf(x_j) \right) 
\end{equation} 
for $i = 1,2,\hdots, n$. This operation can be further lifted to continuous manifold signals $f$ as 
%If we extend $\bbf$ to a continuous function over the manifold $\ccalM$ so as to evaluate the function value difference between the given function values $\{f(x_i)\}_{i=1}^n$, we can denote this discrete Laplace operator as \red{(L: I couldn't understand this passage.)}
\begin{equation}
\label{eqn:discrete_laplacian}
    \bbL^\epsilon_n f(x) =  \sum_{j=1}^n K_\epsilon \left(\frac{\|x-x_j\|^2}{\epsilon}\right) \left( f(x) -f(x_j) \right)
\end{equation}
where $x \in \ccalM$.
If we additionally extend the set of sampled points from $X$ to all of the manifold $\ccalM$, we obtain the functional approximation of the geometric graph Laplacian
\begin{equation}
\label{eqn:functional_laplacian}
    \bbL^\epsilon f(x) =  \int_\ccalM K_\epsilon \left(\frac{\|x-y\|^2}{\epsilon}\right) \left( f(x) -f(y) \right)   \text{d}\mu(y) \text{.}
\end{equation}

%\red{L: Done with this section. Should we merge this one and the next?}

% \section{Geometric Laplacian Convergence} \label{sec:converge_laplacian}

The definition of the weight function $K_\epsilon$ allows the construction of geometric graphs with different levels of sparsity. With different choices of $K_\epsilon$ (i.e., whether $K_\epsilon$ has a bounded support, the relationship between $\epsilon$ and the number of sampling nodes $n$, etc.), the graphs can be sparse (average node degree $\Theta(1)$), relatively sparse (average node degree $\Theta(\log n)$) or dense (average node degree $\Theta(n)$). 
% \red{(L: Is this node degree for all nodes? Or average node degree? Or max node degree?)} 
In the following, we will focus on two kernel definitions that allow constructing dense and relatively sparse geometric graphs.

% \red{L: Add a paragraph saying that this geometric graph construction is very versatile and allows to construct a graphs with varied levels of sparsity based on the kernel $K_\epsilon$, and that in the following sections we discuss two kernel constructions allowing to model dense and sparse graphs.}

\subsection{Laplacian Convergence: Dense Graphs}

%\red{L: The title of this subsection mentions dense graphs, but dense graphs are never discussed in the text. Please rewrite the section (especially the beginning) in one of the following two ways: (a) start by saying that dense graphs can be modeled using the Gaussian kernel, and explain why that is the case---maybe something about the degree; say that the Laplacian of such graphs, and its spectral properties, have been widely studied by Belkin, perhaps mentioning some key findings from his analysis. (b) rename the section ``Laplacian Convergence of Gaussian Geometric Graphs''. In this case, you can leave the section more or less as is and add a comment at the end of the section explaining why graphs constructed with a Gaussian kernel are sparse.}

Dense geometric graphs can be constructed when pairs of points $x_i$ and $x_j$ are connected with a weight function $K_\epsilon$ defined on an unbounded support (i.e. $[0, \infty)$), which connects $x_i$ and $x_j$ regardless of the distance between them, but often with the edge weight inversely proportional to this distance. This results in a dense geometric graph with $n^2$ edges.
In particular, the Gaussian kernel has been widely used to define the weight value function due to the good approximation properties of the corrresponding graph Laplacians vis-\`a-vis the Laplace-Beltrami operator \cite{dunson2021spectral, belkin2008towards}. In the Gaussian case, the weight function $K_\epsilon$ is computed explicitly as 
\begin{equation}
\label{eqn:gauss_kernel}
K_\epsilon\left(\frac{\|x-y\|^2}{\epsilon}\right)= \frac{1}{n}\frac{1}{\epsilon^{d/2+1}(4\pi)^{d/2}} e^{-\frac{\|x-y\|^2}{4\epsilon}},
\end{equation}
with $d$ representing the dimension of the manifold. The consistency of the geometric graph Laplacian constructed with this Gaussian kernel is ensured by a non-asymptotic error bound.
%when operating on the eigenfunction of the LB operator $\ccalL$. 
Explicitly, the following result quantifies the approximation of the dense geometric graph Laplacian in the weak sense.
%\blue{with a point-wise difference bound}. %\red{(L: Mention what $d$ is again.)}

%\red{L: The following proposition does not show convergence of the Laplacian in operator norm; it actually shows weak convergence. This is an important distinction, so make sure that this is clear and that you're not overstating the result.}

\begin{proposition}\label{thm:operator-diff}
Let $\ccalM\subset \reals^{\mathsf{N}}$ be equipped with LB operator $\ccalL$, whose eigendecomposition is given by \eqref{eqn:laplace-decomp}. Let $\bbG_n$ be a dense geometric graph constructed from $n$ points sampled u.i.d. from $\ccalM$ with the edge weights defined as \eqref{eqn:weight} and \eqref{eqn:gauss_kernel}, $\epsilon=\epsilon(n)> n^{-1/(d+4)}$. Then, with probability at least $1-\delta$, the following holds 
\begin{equation}
\label{eqn:operator-dense}
    | \bbL_n^\epsilon \bm\phi_i(x)- \ccalL \bm\phi_i(x)|\leq \left(C_1 \sqrt{ \frac{ \ln(1/\delta)}{2n\epsilon^{d+2}} }+C_2\sqrt{\epsilon}\right) \lambda_i^{\frac{d+2}{4}} \text{.}
\end{equation}
The constants $C_1$, $C_2$ depend on the volume of the manifold {and are defined in Appendix \ref{app:operator}.}
%\red{(L: This theorem statement is too long. Perhaps we can shorten it by removing some unnecessary repeated definitions, such as the eigenspectrum of $\ccalL$.)}
\end{proposition}
\begin{proof}
See Section \ref{app:operator} in supplemental materials.
\end{proof} 
 We can see that the quality of the approximation of $\ccalL$ by $\bbL^{\epsilon}_n$ relates not only to the number of sampling points $n$ but also grows with the corresponding eigenvalue $\lambda_i$. This can be interpreted to mean that eigenfunctions associated with eigenvalues oscillates faster \cite{shi2010gradient}. 

Based on this approximation result and using the Davis-Kahan theorem \cite{seelmann2014notes}, we can further derive a non-asymptotic approximation bound relating the spectra of the dense geometric graph Laplacian and of the LB operator. 
%\red{This result, stated in Theorem \ref{thm:converge-spectrum}, will allows us to analyze the approximation of a manifold filter by a graph filter through their spectral representations. (L: You don't need to say this here. This is a result of independent interest. We will talk about how this implies filter convergence when we get to it in the following section.)}

\begin{proposition}\label{thm:converge-spectrum}
Let $\ccalM\subset \reals^\mathsf{N}$ be equipped with LB operator $\ccalL$, whose eigendecomposition is given by \eqref{eqn:laplace-decomp}. Let $\bbL_n^\epsilon$ be the discrete graph Laplacian of the dense geometric graph $\bbG_n$ defined as in \eqref{eqn:discrete_laplacian} and \eqref{eqn:gauss_kernel}, with spectrum given by $\{\lambda_{i,n}^\epsilon,\bm\phi_{i,n}^\epsilon\}_{i=1}^n$.
Fix $K\in \mathbb{N}^+$ and  $\epsilon=\epsilon(n)> n^{-1/(d+4)}$. Then, with probability at least $1-e^{-2n\epsilon^{d+3}}$, we have 
\begin{equation}
    |\lambda_i-\lambda_{i,n}^\epsilon|\leq \Omega_{1,K} \sqrt{\epsilon}, \quad 
    \|a_i\bm\phi_{i,n}^\epsilon-\bm\phi_i\|\leq \Omega_{2,K} \sqrt{\epsilon}/\theta,
\end{equation}
with $a_i\in \{-1,1\}$ for all $i<K$ and $\theta$ the eigengap of $\ccalL$, i.e., $\theta=\min_{1\leq i\leq K}\{\lambda_i-\lambda_{i-1},\lambda_{i+1}-\lambda_{i}\}$. The constants $\Omega_{1,K}$, $\Omega_{2,K}$ depend on $\lambda_K$, $d$ and the volume of $\ccalM$.
%\red{(L: This thorem statement is also too long. It has many repeated definitions and is quite different than the previous one (e.g., you write $\ccalM \subset \reals^N$ here, and ``let $\ccalM$ be a compact smooth differentiable $d$-dimensional manifold embedded in $\reals^N$'' in Prop. 1). Try to write all theorems in the same format.)}
\end{proposition}
\begin{proof}
See Section \ref{app:spectrum} in supplemental materials.
\end{proof}
Looking at Proposition \ref{thm:converge-spectrum}, we can see that the element-wise non-asymptotic convergence of the eigenvalues and eigenfunctions is only guaranteed in a limited part of the spectrum ($i < K$), and that the upper bound is related to the upper limit $\lambda_K$. This can be interpreted to mean that as eigenfunctions in the high frequency domain oscillate faster, they are harder to approximate. Therefore, when designing filters, we need to be especially careful when trying to discriminate components in the high frequency domain.
%\red{L: Discuss this theorem, even if some of the implications are repetitions of those of Prop. 1. E.g., how do the eigengap and the volume of the manifold affect the bound? When can these effects be problematic?}

\subsection{Laplacian Convergence: Relatively Sparse Graphs}
%\red{L: Rewrite this section in the image of the previous section. E.g., choose the same option (a or b) you chose in the previous section. Parallelism is important.}

We can construct relatively sparse graphs by setting the weight function $K_\epsilon$ with a bounded support, i.e., only nodes that are within a certain distance of one another can be connected by an edge. We consider the following weight function from \cite{calder2019improved}, which has been shown to provide a good approximation of the LB operator,
% The graph Laplacians of sparse graphs can also approximate the LB operator if the kernel function is set as 
\begin{equation}
\label{eqn:compact_kernel}
K_\epsilon\left(\frac{\|x-y\|^2}{\epsilon}\right)= \frac{1}{n}\frac{d+2}{\epsilon^{d/2+1}\alpha_d} \mathbbm{1}_{[0,1]}\left(\frac{\|x-y\|^2}{\epsilon}\right),
\end{equation}
where $\alpha_d$ is the volume of the unit ball in $\reals^d$. As shown by the indicator function, edges only connect points whose distance is within the radius $\sqrt{\epsilon}$. From the theory of random geometric graphs \cite{penrose2003random}, the order of the radius $\epsilon$ decides the order of the average node degree. I.e., if $\epsilon$ is in the order of $\Theta(1)$, the expected node degree is in the order of $\Theta(n)$, which is a dense regime. If $\epsilon$ is in the order of $\Theta((\log(n)/n)^{2/d})$, the expected node degree is in the order of $\Theta(\log(n))$, which is a relatively sparse regime. When setting $\epsilon$ in the order of $\Theta(n^{-2/d})$, the graph is sparse with average node degree $\Theta(1)$.

The following proposition provides a non-asymptotic error bound (in the weak sense) for the approximation of the LB operator by the relatively sparse graph Laplacian.

%\red{L: Where does the kernel above come from? How does it allow to model sparse graphs? Illustrate with a comment about the degree.}

\begin{proposition}{\cite[Theorem~3.3]{calder2019improved}}
\label{thm:operator-diff-sparse}
Let $\ccalM\in \reals^\mathsf{N}$ be equipped with LB operator $\ccalL$, whose eigendecompostion is given by \eqref{eqn:laplace-decomp}. Let $\bbG_n$ be a sparse geometric graph constructed from $n$ points sampled u.i.d. from $\ccalM$ with the edge weights defined as \eqref{eqn:weight} and \eqref{eqn:compact_kernel}, $\epsilon > (\log(n)/2n)^{2/(d+2)}$. Then, with probability at least $1-\delta$, the following holds 
\begin{equation}
\label{eqn:operator-sparse}
    | \bbL_n^\epsilon \bm\phi_i(x)- \ccalL \bm\phi_i(x)|\leq \left( C_3 \sqrt{\frac{\ln{(2n/\delta)}}{cn\epsilon^{d+2}}}+ C_4 \sqrt{\epsilon}\right)\lambda_i^{\frac{d+2}{4}} \text{.}
\end{equation}
The constants $C_3$, $C_4$ depend on the volume of the manifold.
%\red{(L: Reminder to write this in the same format as the other theorems.)}
\end{proposition}

% \red{L: Discuss this theorem. How does the bound in Prop. 3 (sparse case) differ from the bound in Prop. 1 (dense case)?}

Comparing Proposition \ref{thm:operator-diff} with Proposition \ref{thm:operator-diff-sparse}, we see that for large enough $\epsilon$ the difference between the graph Laplacian and the Laplace-Beltrami operator is in the order of $\smash[b]{O(  \sqrt{{\ln(2n/\delta)}/{(n \epsilon^{d+2})}}, \sqrt{\epsilon})}$ in the sparse graph setting, which is slightly larger than the order $\smash[b]{O( \sqrt{{\ln(1/\delta)}/{(n \epsilon^{d+2})} },\sqrt{\epsilon})}$ we observe in the dense graph case \eqref{eqn:operator-dense}. This agrees with the intuition that as dense graphs include more information about the manifold---by taking into account all the connections between each pair of sampling points---, their Laplacian provides a better approximation to the Laplace-Beltrami operator of the underlying manifold than those of relatively sparse graphs.

The differences between the eigenvalues and eigenfunctions can be bounded similarly, based on the Davis-Kahan theorem.

\begin{proposition}{\cite[Theorem~2.4, Theorem~2.6]{calder2019improved}}\label{thm:converge-spectrum-sparse}
Let $\ccalM\subset \reals^\mathsf{N}$ be equipped with LB operator $\ccalL$, whose eigendecomposition is given by \eqref{eqn:laplace-decomp}. Let $\bbL_n^\epsilon$ be the discrete graph Laplacian of graph $\bbG_n$ defined as \eqref{eqn:discrete_laplacian} and \eqref{eqn:compact_kernel}, with spectrum $\{\lambda_{i,n}^\epsilon,\bm\phi_{i,n}^\epsilon\}_{i=1}^n$.
Fix $K\in \mathbb{N}^+$ and assume that  $\epsilon=\epsilon(n)\geq \left({\log(n)}/{2n}\right)^{2/(d+2)}$ %\red{(L: Avoid inline equations that deform the line spacing. There are some LaTeX commands to help with that, or you can use $/$ instead of frac.)}. 
Then, with probability at least $1-2ne^{-cn\epsilon^{d/2+2}}$, we have 
\begin{equation}
    |\lambda_i-\lambda_{i,n}^\epsilon|\leq C_{K,1}\sqrt{\epsilon}, \quad 
    \|a_i\bm\phi_{i,n}^\epsilon-\bm\phi_i\|\leq C_{K,2} \sqrt{\epsilon}/\theta,
\end{equation}
with $a_i\in\{-1,1\}$ for all $i<K$ and $\theta$ the eigengap of $\ccalL$, i.e., $\theta=\min_{1\leq i\leq K}\{\lambda_i-\lambda_{i-1},\lambda_{i+1}-\lambda_{i}\}$. The constants $ C_{K,1}$, $ C_{K,2}$ depend on $\lambda_K$, $d$ and the volume of $\ccalM$.
%\red{(L: Reminder to write this in the same format as the other theorems.)}
\end{proposition}

Akin to the result described in Proposition \ref{thm:converge-spectrum}, the approximation errors of the eigenvalues and eigenfunctions can only be bounded within a certain range of the spectrum ($<\lambda_K$). Comparing with the bounds given in Proposition \ref{thm:converge-spectrum}, the orders of the errors both depend on variable $\epsilon$ in the weight function. Meanwhile, for the convergence probability there is a higher probability in the dense geometric graph setting compared with the relatively sparse geometric graph setting, which also supports the observation that dense geometric graphs can provide a better approximation of the underlying manifold than their sparse counterparts.

\myparagraph{Remark 1}
We note that $\epsilon$ in \eqref{eqn:compact_kernel} controls the connectivity radius of each node to its neighboring nodes, and that the average vertex degree is given by $\alpha_d n \epsilon^{d/2}$ according to the random geometric graph theory \cite{hamidouche2020spectral}. In the case of Proposition \ref{thm:converge-spectrum-sparse}, where $\epsilon$ is fixed in the order of $(\log(n)/2n)^{2/(d+2)}$, the average degree scales with an order between $O((\log(n))^2/ n)$ and $O(\log(n))$. Hence, the graphs in Propositions \ref{thm:operator-diff-sparse} and \ref{thm:converge-spectrum-sparse} are relatively sparse.

% \red{L: Discuss the implications of this theorem, even if they are repeated. What are the differences between this proposition and Prop. 2?}

\section{Geometric GNN Convergence}
\label{sec:converge_gnn}
Equipped with Propositions \ref{thm:converge-spectrum} and \ref{thm:converge-spectrum-sparse}, which related the eigenvalues and eigenfunctions of the manifold and graph Laplacian operators, we will next show that convolutional filters operating on dense or sparse geometric graphs constructed by sampling the underlying manifold give good approximations of manifold filters.
%by proving a non-asymptotic approximation error bound.
As neural networks are cascading structures composing convolutional filters, GNNs inherit this approximation property from graph filters, and thus provide good approximations of MNNs. We begin this section by discussing how the definitions of manifold filters and MNNs can be generalized to the sampled geometric graphs. 

%\red{(L: This paragraph is good, but it has to be adapted to talk about what we now call geometric graphs.)}

% Equipped with theoretical approximation results for both the eigenvalues and eigenfunctions of the Laplacian operators, we can now prove that graph filters can approximate manifold filters well in the spectral domain. We first show that we can generalize the definition of manifold convolutional filters to sampled manifolds using the same impulse response $\tilde{h}$. 

\subsection{Geometric Graph Convolution}

If we fix the impulse response function $\tilde{h}(t)$, the definition of manifold filtering in \eqref{eqn:convolution-conti} indicates that the convolution operation on the manifold is parametric with respect to the LB operator. Therefore, we can replace the LB operator, acting on the continuous manifold signal, by the discrete graph Laplacian $\bbL_n^\epsilon$ acting on a geometric graph signal as defined in \eqref{eqn:graph_laplacian}. Explicitly,
\begin{equation}
\label{eqn:graph_filter}
    \bbg = \int_0^\infty \tilde{h}(t) e^{-t\bbL_n^\epsilon} \bbf \text{d}t :=\bbh(\bbL_n^\epsilon)\bbf, \quad \bbg, \bbf \in \reals^n.
\end{equation}
This can be interpreted as a discrete geometric graph filtering process in \textit{continuous-time}, where the exponential term $e^{-\bbL_n^\epsilon}$ should be seen as the GSO. This is slightly different than the graph convolutional filtering $\bbh_\bbG$ defined in \eqref{eqn:graph_convolution}, where we assumed a \textit{discrete-time} frame.% \red{(L: Be more mathematically precise with the previous explanation.)}. 

%\red{(L: Break different topics into paragraphs; avoid long paragraphs.)}
Leveraging \eqref{eqn:spectral-filter}, the spectral representation of the above continuous time geometric graph filter can be written as
\begin{equation}
   \bbg  = \sum_{i=1}^n \hat{h}(\lambda_{i,n}^\epsilon)\langle\bbf, \bm\phi_{i,n}^\epsilon \rangle_{L^2(\bbG_n)}\bm\phi_{i,n}^\epsilon,
\end{equation}
where $\{\lambda_{i,n}^\epsilon, \bm\phi_{i,n}^\epsilon\}_{i=1}^n$ is the spectrum of $\bbL_n^\epsilon$. The spectral representation exposes the total dependency of the filter frequency representation on the eigenspectrum of the Laplacian operator This indicates that the relationship between geometric graph filters and manifold filters can be established in the spectral domain using Propositions \ref{thm:converge-spectrum} and \ref{thm:converge-spectrum-sparse}. 

\subsection{Geometric Graph Convolution Convergence}

% \red{L: The theorem statements in this subsection and subsection D are still too long. Have a look at what Fernando did in his stability paper, e.g., Theorem 3: ``With the same hypotheses and definitions of
% Thm. 2 ...''}

Recall that the pointwise convergence of the eigenspectrum in Proposition \ref{thm:converge-spectrum} and \ref{thm:converge-spectrum-sparse} is restricted to a certain spectral range. Yet, the frequency representation of the graph filter has a dependency on all spectral components. Hence, the infinite spectrum of the LB operator inevitably presents a challenge in the convergence analysis of geometric graph filters.  %\red{(L: This paragraph is good overall, but please make the previous sentence more clear and proofread for typos.)}
%This is due to the fact that larger oscillations of the eigenfunctions in the high frequency domain lead to larger approximation errors in the spectrum as Proposition \ref{thm:converge-spectrum} and Proposition \ref{thm:converge-spectrum-sparse} have shown. %\red{(L: Point to the theorems where you showed this fact (that larger eigenvalues are associated with larger approximation errors.)}.

To address this issue, we exploit Weyl's law \cite{arendt2009mathematical}.
%to help tackle the high-frequency spectrum of the LB operator $\ccalL$. %\red{(L: You need to explain Weyl's law in more detail, or to include it here as a lemma. I prefer the latter.)}. 
This classical result states that the eigenvalues $\{\lambda_i\}_{i=1}^\infty$ of $\ccalL$ grow proportionally to $i^{2/d}$. This indicates that large eigenvalues, in the high frequency domain, tend to accumulate; and that the differences between neighboring eigenvalues tend to become smaller as the eigenvalues grow larger. This phenomenon is formally described in the following lemma. 

\begin{lemma}{\cite[Proposition~3]{wang2022stability}}
Consider a $d$-dimensional manifold $\ccalM\subset \reals^\mathsf{N}$ and let $\ccalL$ be its LB operator with eigenvalues $\{\lambda_k\}_{k=1}^\infty$. Let $C_1$ be an arbitrary constant and $\alpha_d$ the volume of the $d$-dimensional unit ball. Let $\text{Vol}(\ccalM)$ denote the volume of manifold $\ccalM$. For any $\alpha > 0$ and $d>2$, there exists $N_1$,
\begin{equation}
    N_1=\lceil (\alpha d/C_1)^{d/(2-d)}(C_d \text{Vol}(\ccalM))^{2/(2-d)} \rceil
\end{equation}
such that, for all $k>N_1$, 
$\lambda_{k+1}-\lambda_k\leq \alpha$.
\end{lemma}
Equipped with this fact, we can employ a partitioning strategy to separate the infinite spectrum into finite intervals as shown in Definition \ref{def:alpha-spectrum}.

% Unlike the finite spectrum of the graph Laplacian, the LB operator $\ccalL$ possesses an infinite spectrum. We consider Weyl's law \cite{arendt2009weyl} when analyzing the spectral properties of $\ccalL$. This classical result states that the eigenvalues $\lambda_i$ of $\ccalL$ grow in the order of $i^{2/d}$. Therefore, the difference between neighboring eigenvalues becomes quite small in the high frequency spectrum, i.e., large eigenvalues accumulate. Leveraging this fact, we can use a partition strategy to separate the spectrum into finite groups as in Definition \ref{def:alpha-spectrum}.

%%%%%%%%%%%%%%%%%%%%%%%%%%%%%%%%%%%%%%%%%%%%%%%%
%%%%%%%%%%%%%%%%%% DEFINITION %%%%%%%%%%%%%%%%%% 
%%%%%%%%%%%%%%%%%%%%%%%%%%%%%%%%%%%%%%%%%%%%%%%%

\begin{definition}\cite[Definition~4]{wang2022stability} ($\alpha$-separated spectrum)\label{def:alpha-spectrum}
The $\alpha$-separated spectrum of the LB operator $\ccalL$ is defined as a partition $\Lambda_1(\alpha) \cup \ldots\cup \Lambda_N(\alpha)$ satisfying
$|\lambda_i - \lambda_j| > \alpha$ for $\lambda_i \in \Lambda_k(\alpha)$ and $\lambda_j \in \Lambda_l(\alpha)$, $k \neq l$.
\end{definition}

 The $\alpha$-separated spectrum as defined above can be achieved by means of a $\alpha$-FDT filter, which is defined as follows.

%%%%%%%%%%%%%%%%%%%%%%%%%%%%%%%%%%%%%%%%%%%%%%%%
%%%%%%%%%%%%%%%%%% DEFINITION %%%%%%%%%%%%%%%%%% 
%%%%%%%%%%%%%%%%%%%%%%%%%%%%%%%%%%%%%%%%%%%%%%%%

\begin{definition}\cite[Definition~5]{wang2022stability} ($\alpha$-FDT filter)\label{def:alpha-filter}
The $\alpha$-frequency difference threshold ($\alpha$-FDT) filter is defined as a filter $\bbh(\ccalL)$ whose frequency response satisfies
\begin{equation} \label{eq:fdt-filter}
    |\hhath(\lambda_i)-\hhath(\lambda_j)|\leq \gamma_k \mbox{ for all } \lambda_i, \lambda_j \in \Lambda_k(\alpha) 
\end{equation}
with $\gamma_k\leq \gamma$ for some $\gamma>0$ and $k=1, \ldots,N$. 
\end{definition}

To prove convergence, we will also an assumption on the continuity of the manifold filter, which needs to be a \textit{Lipschitz filter} as defined below.

%%%%%%%%%%%%%%%%%%%%%%%%%%%%%%%%%%%%%%%%%%%%%%%%
%%%%%%%%%%%%%%%%%% DEFINITION %%%%%%%%%%%%%%%%%% 
%%%%%%%%%%%%%%%%%%%%%%%%%%%%%%%%%%%%%%%%%%%%%%%%
\begin{definition}(Lipschitz filter) \label{def:lipschitz}
A filter is $A_h$-Lispchitz if its frequency response is Lipschitz continuous with Lipschitz constant $A_h$,
\begin{equation}
    |\hhath(a)-\hhath(b)| \leq A_h |a-b|\text{ for all } a,b \in (0,\infty)\text{.}
\end{equation}
\end{definition}

Letting the geometric graph filter \eqref{eqn:graph_filter} be a Lipschitz continuous $\alpha$-FDT filter, %we can remove the bandlimited restriction involved in Theorem \ref{thm:converge-spectrum} and Theorem \ref{thm:converge-spectrum-sparse} by giving the eigenvalues that accumulate in high frequency domains similar frequency responses. Therefore, 
we are ready to prove an approximation error bound on the difference between the outputs of a manifold filter and a geometric graph filter operating on a dense geometric graph.
%\red{(L: First introduce the theorem, then discuss how its assumptions are different than those needed for Theorems \ref{thm:converge-spectrum} and \ref{thm:converge-spectrum-sparse}. Also, I don't think you have used the terminology ``bandlimited'' before in the paper. Make sure you explain what you mean by that.)}

%\red{L: Somewhere in the previous paragraphs, you have to mention that the following result is the filter convergence theorem for \emph{dense} graphs.}
% Equipped with the above requirements for the filter frequency response, we can finally establish the upper bound on the manifold filter approximation error on the sampled manifold.

\begin{theorem}(Convergence of filters on dense geometric graphs)
\label{thm:converge-MF-dense}
Let $\ccalM\subset \reals^\mathsf{N}$ be equipped with LB operator $\ccalL$, whose eigendecomposition is given by \eqref{eqn:laplace-decomp}. Let $\bbL_n^\epsilon$ be the discrete graph Laplacian of the dense graph $\bbG_n$ defined as in \eqref{eqn:discrete_laplacian} and \eqref{eqn:gauss_kernel}, with spectrum given by $\{\lambda_{i,n}^\epsilon,\bm\phi_{i,n}^\epsilon\}_{i=1}^n$.
Fix $K\in \mathbb{N}^+$ and assume that  $\epsilon=\epsilon(n)>n^{-1/(d+4)}$ %\red{(L: Avoid inline equations that deform the line spacing. There are some LaTeX commands to help with that, or you can use $/$ instead of frac.)}. 
Let $\bbh(\cdot)$ be the convolutional filter. %parameterized by the LB operator $\ccalL$ of the manifold $\ccalM$ \eqref{eqn:convolution-conti} or by the discrete graph Laplacian operator $\bbL_n^\epsilon$ of the graph $\bbG_n$. 
Under the assumption that the frequency response of filter $\bbh$ is Lipschitz continuous and $\alpha$-FDT with $\alpha^2 \gg \epsilon$, $\alpha >C_{\ccalM,d}K^{2/d-1}$ and $\gamma = \Omega_{2,K}\sqrt{\epsilon}/\alpha$, with probability at least $1-2n^{-2}$ it holds that
\begin{align}\label{eqn:appro_filter}
&\nonumber  \nonumber \|\bbh(\bbL_n^\epsilon)\bbP_n f - \bbP_n\bbh( \ccalL)f\|_{L^2(\bbG_n)}\\&\qquad \quad \qquad 
\leq
 \left(\frac{ N\Omega_{2,K}}{\alpha} +A_h \Omega_{1,K} \right)\sqrt{\epsilon}+C_{gc}\sqrt{\frac{\log n}{{n}}}
\end{align}
where $N$ is the partition size of the $\alpha$-FDT filter and $C_{gc}$ depends on both $d$ and the volume of $\ccalM$. %\red{(L: This theorem statement is too long. Please shorten it, and look at my comments in the previous theorems (avoiding repeated definitions, etc.).)}
\end{theorem}
\begin{proof}
See Appendix \ref{app:nn}. 
\end{proof}

From this theorem, we can see that, if we take $\epsilon = n^{-1/(d+4)}$, the difference between filtering on the constructed graph and on the manifold scales with $O(n^{-1/(2d+8)})$. Therefore, given enough sampling points, one can guarantee good approximation accuracy with high probability. Also note that a higher dimension leads to a larger $\epsilon$, which results in a larger approximation error. This indicates that it is more difficult to approximate manifolds with higher dimension.  

Observe that the partition of the spectrum by the $\alpha$-FDT filter lifts the limitation on the spectrum required in Propositions \ref{thm:converge-spectrum} and \ref{thm:converge-spectrum-sparse}. This is achieved by setting $\alpha$ large enough that eigenvalues larger than $\lambda_K$ are grouped. A smaller $K$ leads to a larger $\alpha$, to ensure that more eigenvalues in the high frequency domain are grouped. Therefore, we can alleviate the divergence of spectral components associated with large eigenvalues by giving them similar frequency responses, which subsequently leads to similar filter outputs. When $K$ gets larger, the constants $\Omega_{1,K}$ and $\Omega_{2,K}'$ scale with $\lambda_K$ as Propositions \ref{thm:converge-spectrum} and \ref{thm:converge-spectrum-sparse} show. Meanwhile, $\alpha$ can be set smaller, and the approximation error becomes larger. 

%\red{L: Perhaps add another paragraph here explaining the effect of the different constants in the bound (other than $\alpha$, which you discuss below)? E.g., how does the manifold dimension affect the error bound?}

%\red{(L: This paragraph talked about two different implications and was too long, so I split it in two.)} 
Another observation we can make from this non-asymptotic error bound is that, for a fixed number of sampling points $n$, an interesting trade-off arises between the approximative and discriminative capabilities of the filter. A larger $\alpha$ (i.e., a smaller $K$) means a less discriminative filter as more eigenvalues tend to be grouped and treated similarly with an almost constant frequency response. A larger $\alpha$ leads to a smaller number of partitions, as the number of singletons decreases, which results in the decrease of the error bound in \eqref{eqn:appro_filter}. This implies that the error bound becomes smaller, as the errors brought by the approximation of the eigenvalues are alleviated by treating different eigenvalues with similar frequency responses.  

For relatively sparse graphs, we can also establish an approximation error bound between the outputs of manifold filters and geometric graph filters \eqref{eqn:graph_filter}.
%\red{L: Add a paragraph introducing the next result and mentioning that it applies to sparse graphs.}

\begin{theorem}(Convergence of filters on relatively sparse geometric graphs)
\label{thm:converge-MF-sparse}
Let $\ccalM\subset \reals^\mathsf{N}$ be equipped with LB operator $\ccalL$, whose eigendecomposition is given by \eqref{eqn:laplace-decomp}. Let $\bbL_n^\epsilon$ be the discrete graph Laplacian of  $\bbG_n$ defined as in \eqref{eqn:discrete_laplacian} and \eqref{eqn:compact_kernel}.
Fix $K\in \mathbb{N}^+$ and assume that $\epsilon=\epsilon(n)>({\log(n)}/{n})^{1/d}$
Let $\bbh(\cdot)$ be the convolutional filter. %parameterized by the LB operator $\ccalL$ of the manifold $\ccalM$ \eqref{eqn:convolution-conti} or by the discrete graph Laplacian operator $\bbL_n^\epsilon$ of the graph $\bbG_n$. 
Under the assumption that the frequency response of the filter $\bbh$ is Lipschitz continuous and $\alpha$-FDT with $\alpha^2 \gg \epsilon$,  $\alpha >C_{\ccalM,d}K^{1-2/d}$ and $\gamma = C_{K,2}{\epsilon}/\alpha$, with probability at least $1-2n\exp(-Cn\epsilon^{d+4})$ it holds that
\begin{align}\label{eqn:appro_filter}
&\nonumber  \nonumber \|\bbh(\bbL_n^\epsilon)\bbP_n f - \bbP_n\bbh( \ccalL)f\|_{L^2(\bbG_n)}\\&\qquad \qquad  
\leq
 \left(\frac{ NC_{K,2}}{\alpha} +A_h C_{K,1}\right)\sqrt{\epsilon}+C_{gc}\sqrt{\frac{\log n}{{n}}}
\end{align}
where $N$ is the partition size ofthe $\alpha$-FDT filter and $C_{gc}$ depends on both $d$ and the volume of $\ccalM$. 
% \red{(L: This theorem statement is too long. Please shorten it, and look at my comments in the previous theorems (avoiding repeated definitions, etc.).)}
\end{theorem}
\begin{proof}
See Appendix \ref{app:nn}.
\end{proof}

Since the convergence constants in this theorem are related with the convergence rates presented in Proposition \ref{thm:converge-spectrum-sparse}, the convergence of geometric graph filters on sparse graphs is accordant with the convergence of the Laplacian operator on sparse graphs, and so similar comments apply. Namely, when compared with sparse geometric graph filters, dense geometric graph filters provide a better approximation of the underlying manifold filters because dense graphs carry more information about the geometry of the manifold.
%\red{L: Add a discussion similar to the one you have after Theorem 1 (but shorter). Then, conclude with a paragraph highlighting the differences between the results for dense and sparse. E.g., I imagine that convergence is slower for sparse graphs, right? You can talk about how this can be seen in the bounds, and how it makes sense intuitively.}

% From this theorem, we see that if we take $\epsilon = n^{-1/(d+4)}$, the difference between filtering on the manifold and the sampled manifold is in the order of $O(n^{-1/(2d+8)})$. Thus, with a large enough number of sampling points, this approximation provides a good accuracy with high probability. This provides a formal theoretical guarantee for the ability of graph filters to approximate manifold filters on sampled manifolds, enabling approximate linear information processing on continuous non-Euclidean domains.

\subsection{Geometric Graph Neural Network}

As filtering on the constructed geometric graph can successfully approximate filtering on the manifold, we can cascade geometric graph filters and nonlinearities to construct a geometric graph neural network that we expect to provide a good approximation of the manifold neural network. Given the geometric graph filter defined in \eqref{eqn:graph_filter}, the geometric GNN on $\bbG_n$ can be written as
\begin{equation}\label{eqn:dis-mnn}
    \bbx_l^p = \sigma\left(\sum_{q=1}^{F_{l-1}} \bbh_l^{pq}(\bbL_n^\epsilon) \bbx^q_{l-1} \right),
\end{equation}
where $\bbh_l^{pq}(\bbL_n^\epsilon)$ is the filter at the $l$-th layer of this GNN mapping the $q$-th feature in the $l-1$-th layer to the $p$-th feature in the $l$-th layer, with $1\leq q \leq F_{l-1}$ and $1\leq p\leq F_l$. We denote the number of features in the $l$-th layer $F_l$ (we have dropped the subscript $n$ in $\bbx_l^p$ and $\bbx_{l-1}^q$ for simplicity).
%similar to the notations used in the manifold neural network representations. 
Gathering the filter functions in the set $\bbH$, this geometric GNN on $\bbG_n$ can be represented more concisely as the map $\bm\Phi(\bbH, \bbL_n, \bbx)$. 

%\red{L: Rewrite the above with our new geometric graph terminology.}

\subsection{Geometric GNN Convergence}

Imposing a continuity assumption on the nonlinearity function, we can prove the following approximation error bound for geometric GNNs and MNNs.

%%%%%%%%%%%%%%%%%%%%%%%%%%%%%%%%%%%%%%%%%%%%%%%%
%%%%%%%%%%%%%%%%%% ASSUMPTION %%%%%%%%%%%%%%%%%% 
%%%%%%%%%%%%%%%%%%%%%%%%%%%%%%%%%%%%%%%%%%%%%%%%
\begin{assumption}(Normalized Lipschitz nonlinearity)\label{ass:activation}
 The nonlinearity $\sigma$ is normalized Lipschitz continuous, i.e., $|\sigma(a)-\sigma(b)|\leq |a-b|$, with $\sigma(0)=0$.
\end{assumption}
 
We note that this assumption is reasonable, since most common nonlinearity functions (e.g., the ReLU, the modulus and the sigmoid) are normalized Lipschitz.

% At the first layer of the MNN, the input features are the input data $f^q$ for $1\leq q\leq F_0$. At the output of the MNN, the output features are given by the outputs of the $L$-th layer, i.e., $f_L^p$ for $1 \leq p \leq F_L$. To represent the MNN more succinctly, we may gather the impulse responses of the manifold convolutional filters $\bbh_l^{pq}$ across all layers in a function set $\bbH$, and define the MNN map $\bbPhi(\bbH,\ccalL, f)$. This map emphasizes that the MNN is parameterized by both the filter functions and the LB operator $\ccalL$. We next will analyze the stability of $\bbPhi(\bbH,\ccalL,f)$ with respect to perturbations on the underlying manifold.

\begin{theorem}
\label{thm:converge-MNN} 
{With the same hypotheses and definitions as in Theorem \ref{thm:converge-MF-dense} and \ref{thm:converge-MF-sparse} respectively,}
% Let $\ccalM\subset \reals^\mathsf{N}$ be equipped  with LB operator $\ccalL$ whose eigendecomposition is given by \eqref{eqn:laplace-decomp}. Let $\bbL_n^\epsilon$ be the discrete graph Laplacian of the geometric graph $\bbG_n$
% %defined as \eqref{eqn:discrete_laplacian} and \eqref{eqn:compact_kernel}.
% Fix $K\in \mathbb{N}^+$ and assume that $n$ is sufficiently large so that $\epsilon=\epsilon(n)>({\log(n)}/{n})^{1/d}$
% Let $\bbh(\cdot)$ be the convolutional filter parameterized by the LB operator $\ccalL$ of the manifold $\ccalM$ \eqref{eqn:convolution-conti} or by the discrete graph Laplacian operator $\bbL_n^\epsilon$ of the graph $\bbG_n$. Let $\bbh(\cdot)$ be the convolutional filter parameterized by the LB operator $\ccalL$ of the manifold $\ccalM$ \eqref{eqn:convolution-conti} or by the discrete graph Laplacian operator $\bbL_n^\epsilon$ of the graph $\bbG_n$. 
let $\bm\Phi(\bbH,\ccalL,\cdot)$ be an $L$-layer MNN on $\ccalM$ \eqref{eqn:mnn} with $F_0=F_L=1$ input and output features and $F_l=F,l=1,2,\hdots,L-1$ features per layer. Let $\bm\Phi(\bbH,\bbL_n,\cdot)$ be the GNN with the same architecture applied to the geometric graph $\bbG_n$.
If the nonlinearities satisfy Assumption \ref{ass:activation} and the manifold filters satisfy $\| \bbh(\bbL_n^\epsilon)\bbP_n f - \bbP_n\bbh(\ccalL)f \|_{L^2(\bbG_n)}\leq \Delta_{fil,n}$, it holds that
\begin{align}
&\nonumber  \nonumber \|\bm\Phi(\bbH,\bbL_n^\epsilon, \bbP_n f) - \bbP_n\bm\Phi(\bbH, \ccalL ,f ) \|_{L^2(\bbG_n)}
\leq L F^{L-1} \Delta_{fil,n}
\end{align}
with high probability.
% \red{(L: This theorem statement is too long. Please shorten it, and look at my comments in the previous theorems (avoiding repeated definitions, etc.).)}
\end{theorem}
\begin{proof}
See Section \ref{app:mnn} in supplemental materials.
\end{proof}

We conclude that the geometric GNN converges to the MNN as long as the geometric graph filter components give good approximations of the corresponding manifold filters when restricted to the sampled geometric graphs. Further, we see that this error consists of two terms. The first term, $LF^{L-1}$, depends on the number of filters and layers in the MNN architecture. More specifically, the approximation bound grows linearly with the number of layers $L$ and exponentially with the number of features $F$ where the rate is determined by $L$. This term appears because the approximation errors propagate across all the manifold filters in all layers of the MNN. The second term is the approximation error constant of the geometric graph filters $\Delta_{fil,n}$, which should be replaced with the error bounds from Theorem \ref{thm:converge-MF-dense} and Theorem \ref{thm:converge-MF-sparse} for dense and sparse graphs respectively. 
%Based on this, we can state that if the geometric GNNs are constructed according to the assumptions in Theorem \ref{thm:converge-MF-dense} and Theorem \ref{thm:converge-MF-sparse}, then the geometric GNNs converge to the underlying MNNs as the number of nodes grows.

We note that the approximation error constant $\Delta_{fil,n}$ emphasizes that the trade-off between approximation and discriminability is inherited from the convolutional filters. 
%This trade-off is inherited from the geometric graph filters in the layers of the neural network. 
However, since neural networks also have nonlinearities in their layers, the geometric GNN has a better ability to both discriminate high frequency components and approximate the MNN with satisfying accuracy. This is due to the fact that nonlinearities have a spectral mixing effect. In the neural network architecture, the frequency components in the high frequency domain can be shifted to low frequency domain, where they can then be discriminated by the filters in the following layer. This role of nonlinearities has also been discussed in the stability analysis of GNNs \cite{gama2020stability} as well as of MNNs \cite{wang2022stability}.

\subsection{From convergence to transferability}

Leveraging the non-asymptotic convergence results derived in Theorem \ref{thm:converge-MNN}, we can immediately prove a transferability corollary for geometric GNNs constructed from a common underlying MNN. Explicitly, consider two geometric graphs $\bbG_{n_1}$ and $\bbG_{n_2}$ that are either dense \eqref{eqn:gauss_kernel} or realtively sparse graphs \eqref{eqn:compact_kernel}. The difference between the outputs of the same neural network (i.e., with the same weights) on these two graphs is bounded by the following corollary.

\begin{corollary}
\label{cor:transferability}
Let $\ccalM\subset \reals^\mathsf{N}$ {be equipped} with LB operator $\ccalL$. Let $\bbL_{n_1}^\epsilon$ be the discrete graph Laplacian of the graph $\bbG_{n_1}$ defined as in \eqref{eqn:discrete_laplacian} and \eqref{eqn:gauss_kernel} or \eqref{eqn:compact_kernel}. Let $\bbL_{n_2}^\epsilon$ be the discrete graph Laplacian of the graph $\bbG_{n_2}$ constructed in the same manner as $\bbG_{n_1}$. %Both $\bbG_{n_1}$ and $\bbG_{n_2}$ are constructed from $\ccalM$. 
%Let $\bbh(\cdot)$ denote a convolutional filter parameterized by the discrete graph Laplacian.
Let $\bm\Phi(\bbH,\cdot, \cdot)$ be an $L$-layer GNN with $F_0=F_L=1$ input and output features and $F_l=F,l=1,2,\hdots,L-1$ features per layer. 
Assume that the filters and nonlinearity functions are as in Definitions \ref{def:alpha-filter}, \ref{def:lipschitz} and satisfy Assumption \ref{ass:activation} respectively. Then, it holds in high probability that
\begin{align}
&\nonumber  \|\bbI_{n_1} \bm\Phi(\bbH,\bbL_{n_1}^\epsilon, \bbP_{n_1} f) - \bbI_{n_2} \bm\Phi(\bbH,\bbL_{n_2}^\epsilon, \bbP_{n_2} f) \|
\\ & \qquad \qquad \qquad \qquad \qquad\leq L F^{L-1} ( \Delta_{fil,n_1} + \Delta_{fil,n_2} ) \text{.}
\end{align}
\end{corollary}

From this corollary, we can observe that the geometric GNN trained on one geometric graph can be directly transferred to another geometric graph if they are constructed in the same way from the same underlying manifold. The difference between the outputs of the GNNs depends on the size of the neural network architecture as well as on the approximation capability of the filter functions. Larger numbers of sample points $n_1$ and $n_2$ lead to smaller approximation errors $\Delta_{fil,n_1}$ and $\Delta_{fil,n_2}$, that is, better approximation leads to better transferability in geometric GNNs.

The main implication of this corollary is that if we have trained a geometric GNN on a relatively small geometric graph, this GNN can be transferred to another, larger geometric graph with performance guarantees, which is helpful because it is costly to train GNNs on large graphs. Using the bound in Corollary \ref{cor:transferability}, we can determine the minimum number of sampled points needed for training a geometric GNN to meet a given approximation error tolerance. Therefore, the geometric GNN is a suitable model to approximate both continuous MNNs and large geometric GNNs.

%\red{(L: The discussion on the tradeoff, and on how it is alleviated on GNNs, has to be much more detailed. Have a look at what we did in your previous paper, and to the related discussions in my transferability paper.)}

% This term arises due to the propagation of the underlying operator perturbations across all the manifold filters in all layers of the MNN. The second term is $\pi N/\alpha$ or $\pi M/\gamma$, which results from the deviations of the eigenfunctions as well as from the frequency response variations within the same eigenvalue partition. Finally, the third term, $A_h$ or $B_h$, is given by the Lipschitz or integral Lipschitz constants which are decided during the filter design or the training process. It is important to note that the stability constant $C_{per}$ brings along the trade-off between stability and discriminability. However, unlike manifold filters, MNNs can be both stable and discriminative. This arises from the effects of nonlinear activation functions, as we discuss in further detail in Section \ref{subsec:discussion}.

\section{Discussions}
\label{subsec:discussion}
%!TEX root = stability_manifold_TSP.tex

%%%%%%%%%%%%%%%%%%%%%%%%%%%%%%%%%%%%%%%%%%%%%%%%
%%%%%%%%%%%%%%%%%% SUBSECTION %%%%%%%%%%%%%%%%%% 
%%%%%%%%%%%%%%%%%%%%%%%%%%%%%%%%%%%%%%%%%%%%%%%%
%\red{L: The way this paragraph is written now, it sounds like the stability-discriminability tradeoff only stems from the term that depends on the Lipschitz constant. However, the tradeoff comes from both terms: from the first because a larger frequency threshold leads to filters that give similar response to a large group of eigenvalues; and from the second because smaller Lipschitz constants lead to less discriminative filters. My suggestion is that you don't break up the analysis in the first and second terms. Just discuss the effect of $\alpha,\gamma$ and $A_h,B_h$ on both stability and discriminability, and explain how these effects are opposites and thus lead to a stability-discriminability tradeoff.}
{\myparagraph{Approximation vs. discriminability tradeoff}}
In the convergence results for geometric graph filters on both dense graphs (Theorem \ref{thm:converge-MF-dense}) and sparse graphs (Theorem \ref{thm:converge-MF-sparse}), the approximation error bounds depend on the size of the partition ($N$), the frequency partition threshold ($\alpha$), the Lipschitz continuity constant ($A_h$) and the weight function parameter ($\epsilon$), and the number of nodes in the geometric graph ($n$). The size of the partition and the frequency partition threshold have a concerted effect in the approximation, as a larger frequency partition threshold leads to a larger number of intervals containing more than one eigenvalue and a smaller number of intervals containing only one eigenvalue. The size of the partition however tends to stay the same or decrease because the number of intervals cannot exceed the number of eigenvalues. Therefore, a larger frequency partition threshold results in a smaller partition size. Combined, these lead to a smaller approximation error bound. Meanwhile, a larger frequency threshold also means that a larger number of eigenvalues are grouped and treated with similar frequency responses. This makes the $\alpha$-FDT filters less discriminative. A similar story holds for the Lipschitz constant $A_h$, with smaller Lipschitz constant increasing the approximation error bound but decreasing discriminability. A larger $\epsilon$ leads to a larger approximation error bound, but at the same time the variation range $\gamma$ in the $\alpha$-FDT filter becomes larger which results in better discriminability. In conclusion, there is a trade-off between the approximation and discriminability properties of both geometric graph filters and geometric GNNs, but in GNNs this trade-off is alleviated thanks to the nonlinearities, as we discuss next.

%%%%%%%%%%%%%%%%%%%%%%%%%%%%%%%%%%%%%%%%%%%%%%%%
%%%%%%%%%%%%%%%%%% SUBSECTION %%%%%%%%%%%%%%%%%% 
%%%%%%%%%%%%%%%%%%%%%%%%%%%%%%%%%%%%%%%%%%%%%%%%
\myparagraph{Graph filters vs. GNNs} 
Due to the aforementioned trade-off between approximation and discriminability arising from the spectrum partition, geometric graph filters cannot simultaneously discriminate high-frequency components and provide good approximations of manifold filters. In geometric GNNs, this trade-off is however alleviated by the use of nonlinearities. When processing the outputs of the filters in the previous layer, the nonlinearities in a given layer have the ability to shift some high frequency components to the low frequency domain, where they can then be discriminated by the filters in the following layer. As a consequence of this frequency scattering behavior, geometric GNNs can be simultaneously approximative and discriminative.

%%%%%%%%%%%%%%%%%%%%%%%%%%%%%%%%%%%%%%%%%%%%%%%%
%%%%%%%%%%%%%%%%%% SUBSECTION %%%%%%%%%%%%%%%%%% 
%%%%%%%%%%%%%%%%%%%%%%%%%%%%%%%%%%%%%%%%%%%%%%%%

\myparagraph{Dense graphs vs. sparse graphs} Depending on whether the graphs are dense or sparse, different weight functions are chosen to construct the geometric graph, leading to different convergence regimes. From the results presented in Proposition \ref{thm:operator-diff} and Proposition \ref{thm:converge-spectrum} for dense geometric graphs and Proposition \ref{thm:operator-diff-sparse} and Proposition \ref{thm:converge-MF-sparse} for sparse geometric graphs, we can see that, given the same number of sample points $n$ and the same weight function parameter $\epsilon$, dense graphs have a slightly smaller approximation error bound when compared with sparse graphs---thus implying that geometric graph filters and geometric GNNs provide better approximations of manifold filters and MNNs when operating on dense graphs. This can be explained by the fact that dense graphs include more information; for one, they are complete. However, the sparse graph setting is a more realistic model in practice, and in particular it has been employed in many areas such as wireless communication networks and sensor networks. The approximation guarantees derived in this paper apply in either sparsity regime.

%%%%%%%%%%%%%%%%%%%%%%%%%%%%%%%%%%%%%%%%%%%%%%%%
%%%%%%%%%%%%%%%%%% SUBSECTION %%%%%%%%%%%%%%%%%% 
%%%%%%%%%%%%%%%%%%%%%%%%%%%%%%%%%%%%%%%%%%%%%%%%

\myparagraph{Discretization over time} Theoretically, we employ the definitions of geometric graph convolution and manifold convolution in a continuous-time domain as \eqref{eqn:graph_filter} and \eqref{eqn:convolution-conti} show. However, in practice we need to operate the filters and neural networks in digital systems, which requires the time horizon is discrete. In the following section, we carry out numerical experiments with geometric graph filters defined in a discrete-time domain. Specifically, we discretize the impulse response function with an interval $T_s$ and replace the function with a series of filter coefficients $h_k = \tilde{h}(kT_s)$ \cite{wang2022convolutional}. By fixing the time horizon with finite $K_t$ samples, we can write the practical geometric filter as 
\begin{equation}
\label{eqn:discrete_geometric_filter}
  \bbg =  \sum_{k=0}^{K_t-1} h_k e^{-k\bbL_n^\epsilon}\bbf,\quad \bbf,\bbg\in\reals^n,
\end{equation}
which recovers the graph convolution definition in \eqref{eqn:graph_convolution} with $e^{-\bbL_n^\epsilon}$ seen as the graph shift operator.

\section{Numerical Experiments}
\label{sec:simu}
%!TEX root = conference-stability.tex

%%%%%%%%%%%%%%%%%%%%%%%%%%%%%%%%%%%%%%%%%%%%%%%%
%%%%%%%%%%%%%%%%%% FIGURE %%%%%%%%%%%%%%%%%%%%%% 
%%%%%%%%%%%%%%%%%%%%%%%%%%%%%%%%%%%%%%%%%%%%%%%%
% \begin{figure*}[ht!]
% \centering
% \includegraphics[width=0.2\textwidth]{graph on manifold/dataset.pdf}
% \includegraphics[width=0.2\textwidth]{graph on manifold/traj_generated0.15epoch2900.pdf}  
% \includegraphics[width=0.2\textwidth]{graph on manifold/dataset.pdf} 
% \includegraphics[width=0.2\textwidth]{graph on manifold/traj_generated0.15epoch2900.pdf}  

% \caption{Point cloud models with 300 sampling points in each model. Our goal is to identify chair models from other models such as toilet and table. }
% \label{fig:points}
% \end{figure*}
\subsection{Navigation control}
\begin{figure*}[t!]
     \centering
     \begin{subfigure}[b]{0.2\textwidth}
         \centering
         \includegraphics[width=1.2\textwidth]{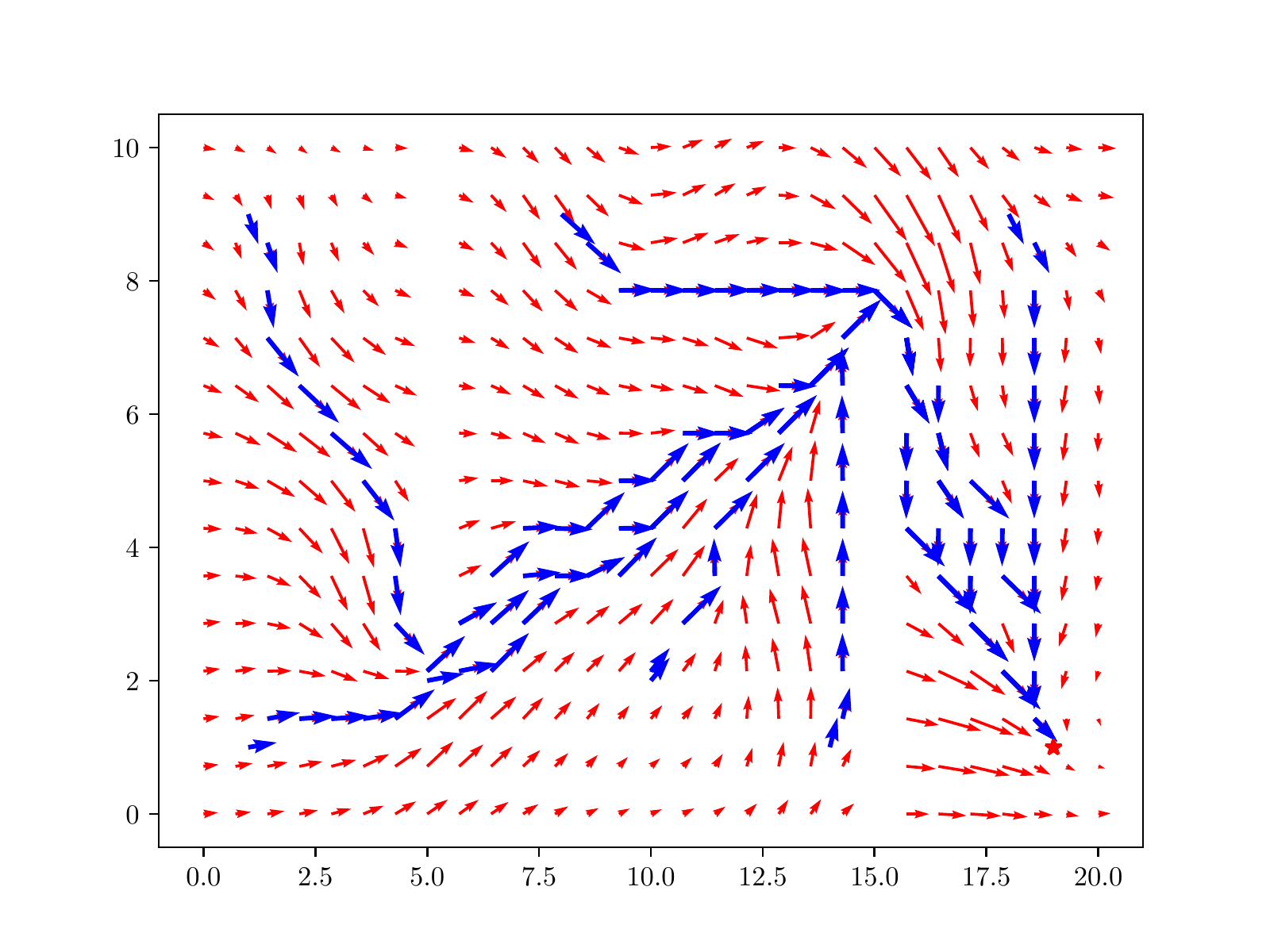}
         \caption{Testing on $n=435$}
         \label{fig:train400}
     \end{subfigure}
    \hspace{1.5em}
     \begin{subfigure}[b]{0.2\textwidth}
         \centering
         \includegraphics[width=1.2\textwidth]{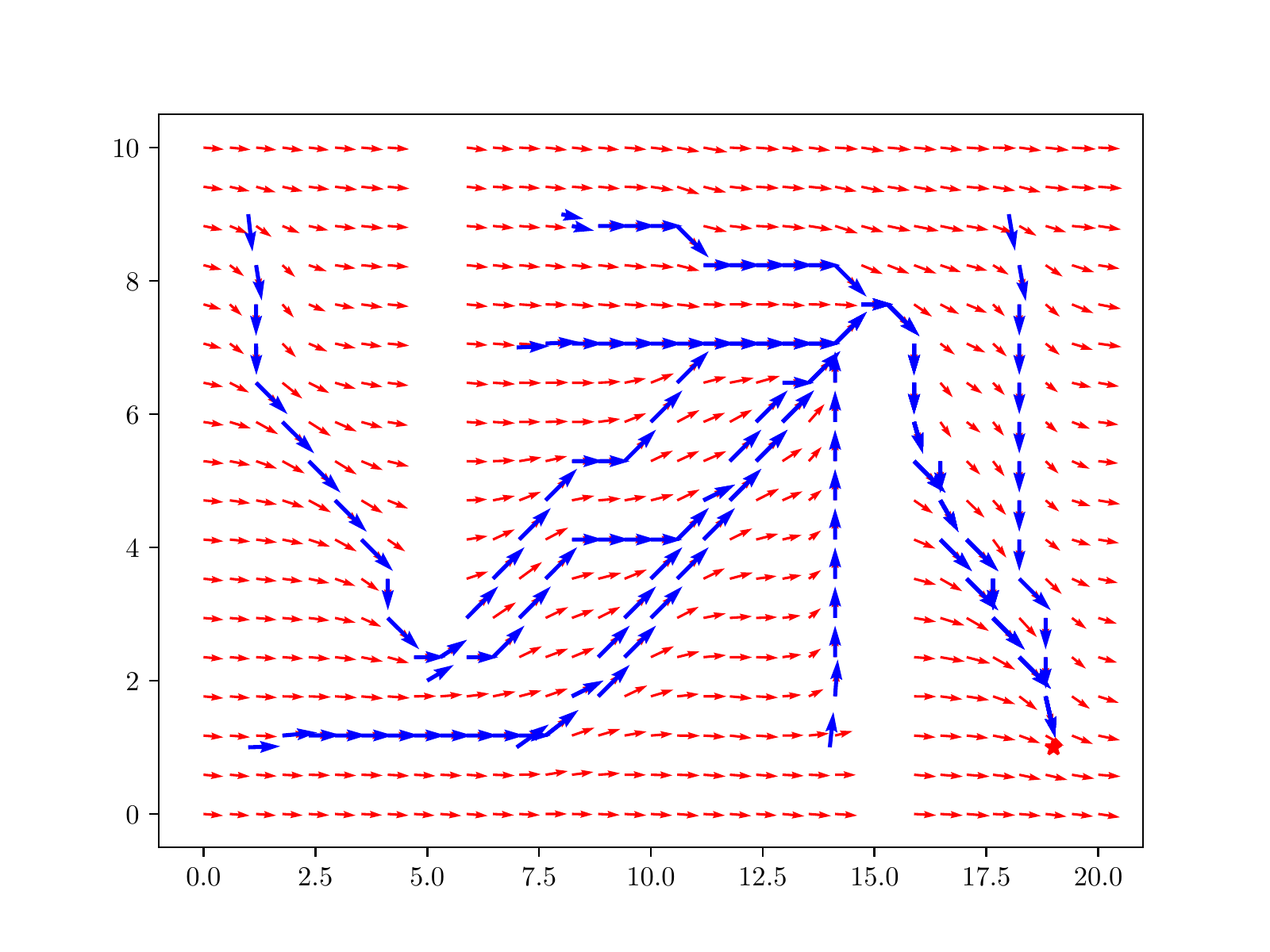}
         \caption{Testing on $n=630$}
         \label{fig:learned400}
     \end{subfigure}
     \hspace{1.5em}
   \begin{subfigure}[b]{0.2\textwidth}
         \centering
         \includegraphics[width=1.2\textwidth]{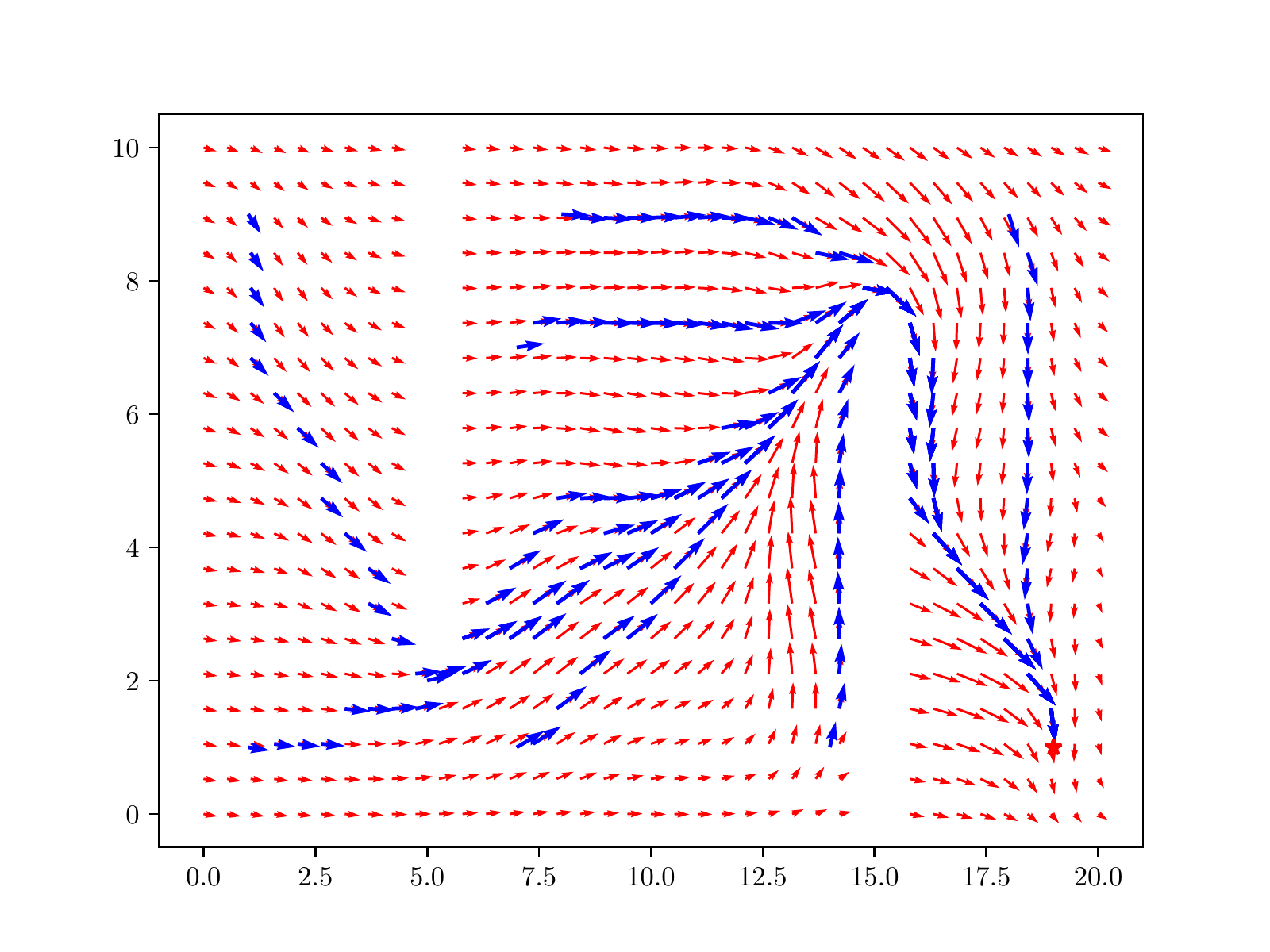}
         \caption{Testing on $n=780$}
         \label{fig:train1000}
     \end{subfigure}
    \hspace{1.5em}
     \begin{subfigure}[b]{0.2\textwidth}
         \centering
         \includegraphics[width=1.2\textwidth]{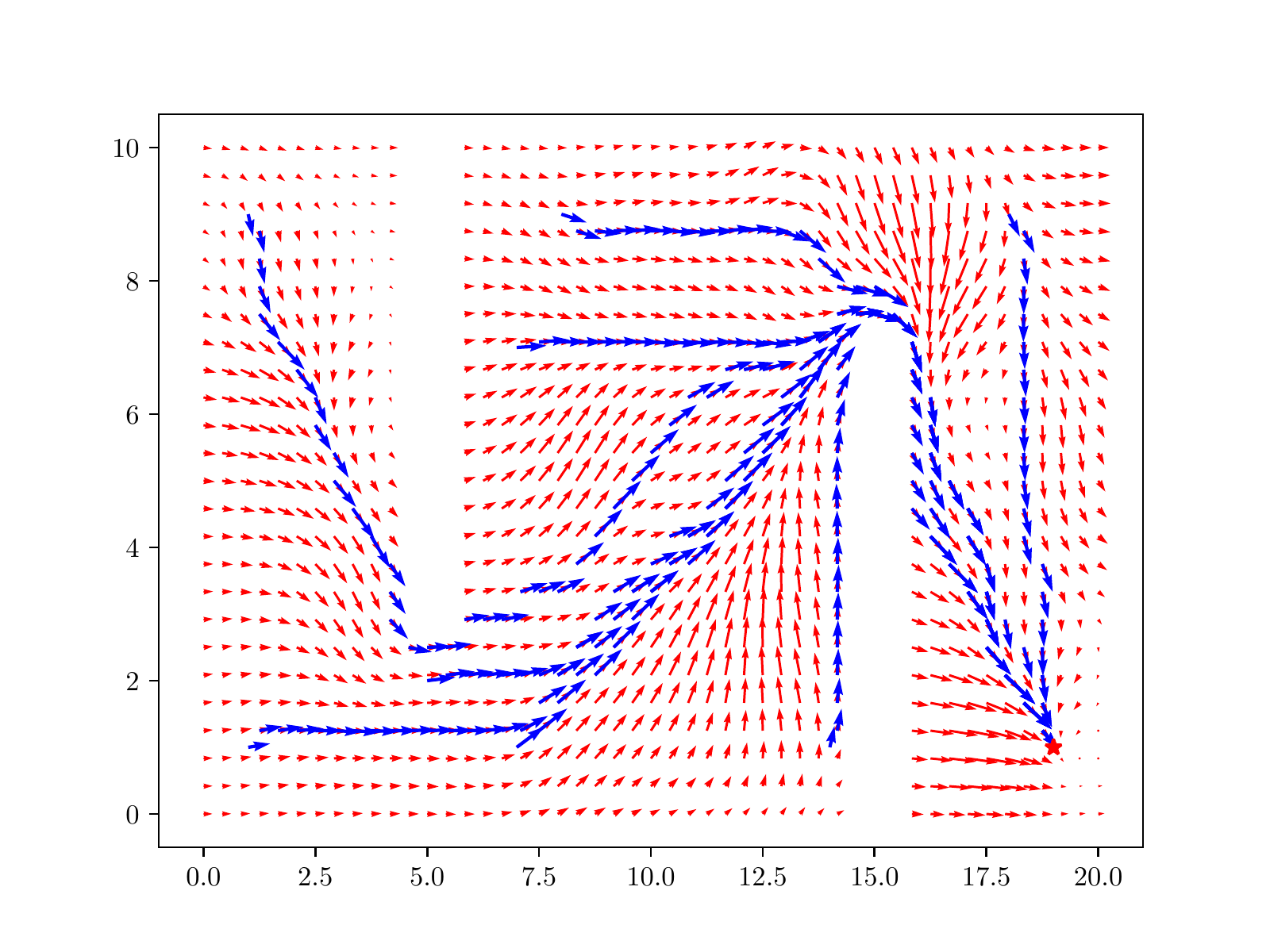}
         \caption{Testing on $n=1225$}
         \label{fig:learned1000}
     \end{subfigure}
   
        \caption{Lipschitz GNN trained on geometric graph with $n=435$ and tested on geometric graphs with $n=435, 630, 780, 1225$. The red arrows in the figures show the predicted directions for all sampled points over the space. }
        \label{fig:direction}
\end{figure*}
We consider a problem of automatic navigation of an agent \cite{cervino2022learning}. We intend to navigate an agent starting from a given point to find a path to the goal without colliding into the obstacles. We sample grid points uniformly over the free space while avoiding the obstacle areas to construct a geometric graph structure involving the topological information.  We first generate several trajectories with Dijkstra's shortest path algorithm that leads the agent from a randomly selected starting point to the goal. Every point along the generated trajectories is labeled with a 2-d direction vector of the velocity. The geometric adjacency matrix is calculated based on the weight function defined in \eqref{eqn:weight} with \eqref{eqn:gauss_kernel}. Specifically, the underlying topology information is captured by setting the Euclidean distance between two points as infinity (i.e. weight value as zero) if there is no direct path connecting these two points without colliding into the obstacles. The input graph signal is the position coordinates of some point over the space and the final output is the direction that can lead to the goal point. The learned architectures are tested via randomly generating $100$ starting points and computing the trajectory by predicting directions of the points along the trajectory. If the trajectory can reach the goal point without colliding into the obstacles, the trajectory is marked as ``success". The performances are measured by calculating the successful rates of all the learned architectures.

\myparagraph{Learning architectures and experiment settings}
We build dense geometric graphs as the approximations to the models as the regular grid graph structure brings little difference between the dense and sparse graph setting. We set the position of each point as input signals on the graphs. The weights of the edges are calculated based on the Euclidean distance between the nodes and the weight function is determined as \eqref{eqn:gauss_kernel} with $\epsilon = 0.2$. We calculate the Laplacian matrix for the graph as the input graph shift operator. We train and test three architectures, including 2-layer Graph Filters (GF), 2-layer Graph Neural Network (GNN) and 2-layer Lipschitz Graph Neural Network (Lipschitz GNN) which all contain $F_0 = 2$ input features, $F_1= 128$ and $F_2=64$ features with $K_t=10$ filters in each layer. ReLU is used as the nonlinearity function in GNN and Lipschitz GNN. In Lipschitz GNN architecture, we put the Lipschitz continuity restriction to the graph filter function by importing a penalty term which is the scaled derivative of the filter function $C_L h'(\lambda)$ to the loss function. A larger Lipschitz constant $C_L$ indicates a smoother filter function. All architectures also include a linear readout layer to map the 2-d direction vector outputs. The loss function is MSE loss and the optimizer is SGD with the learning rate set as $0.0002$. The architectures are trained over $30,000$ epochs.

\myparagraph{Convergence analysis}
The convergence is tested by calculating the difference of the outputs of the graph filters in the last layer for each architecture. By training the graph filters and GNNs on geometric graphs with $n= 190, 276, 435, 630, 780 $, and plotting the differences between the output of the trained graph filters or GNNs on geometric graphs with size $n$ and on the relatively large enough geometric graphs with size $n=1225$. From Figure \ref{fig:diff_1layer} we can see that the output differences  decrease and converge as the trained geometric graph size grows, which is accordant with Corollary \ref{cor:transferability}. This also verifies the convergence of geoemtric GNNs as presented in \ref{thm:converge-MNN} if we see the large geometric graph as an approximation of the underlying manifold. We can conclude that Lipschitz GNNs have better approximations while GNNs outperform Graph Filters.
\begin{figure}[h!]
         \centering
       \includegraphics[width=0.47\textwidth]{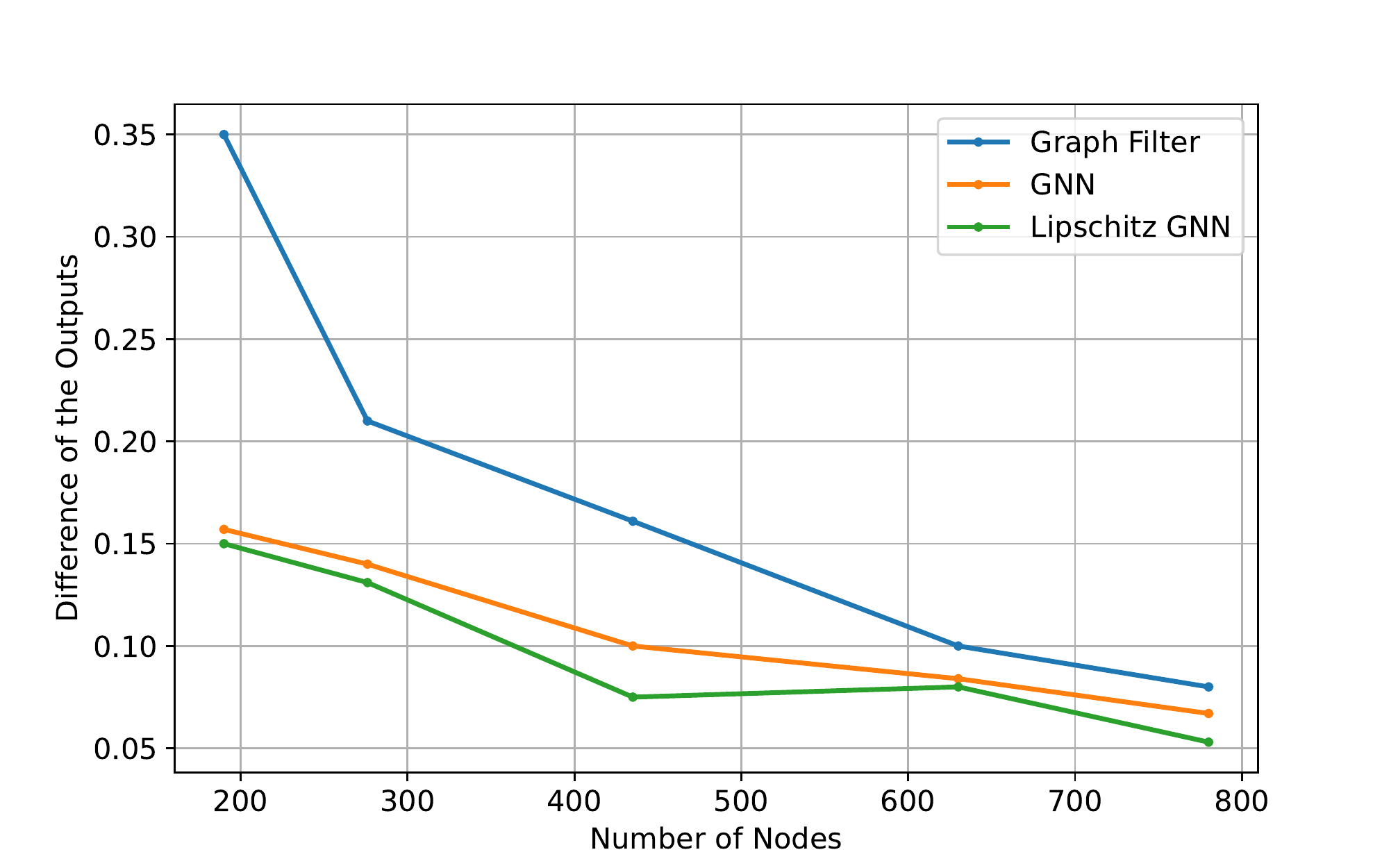}
         \caption{Differences of the outputs of 2-layer graph filters, 2-layer GNN and 2-layer Lipschitz GNN.}
         \label{fig:diff_1layer}
\end{figure}

We further study the effect of Lipschitz continuity of the graph filter functions by changing the penalty constant $C_L$. From Figure \ref{fig:diff_1layer-penalty} we can see as the constant grows, the output differences become smaller, which attests our claim that smoother filter functions provide better approximations. 
\begin{figure}[h!]
         \centering
       \includegraphics[width=0.47\textwidth]{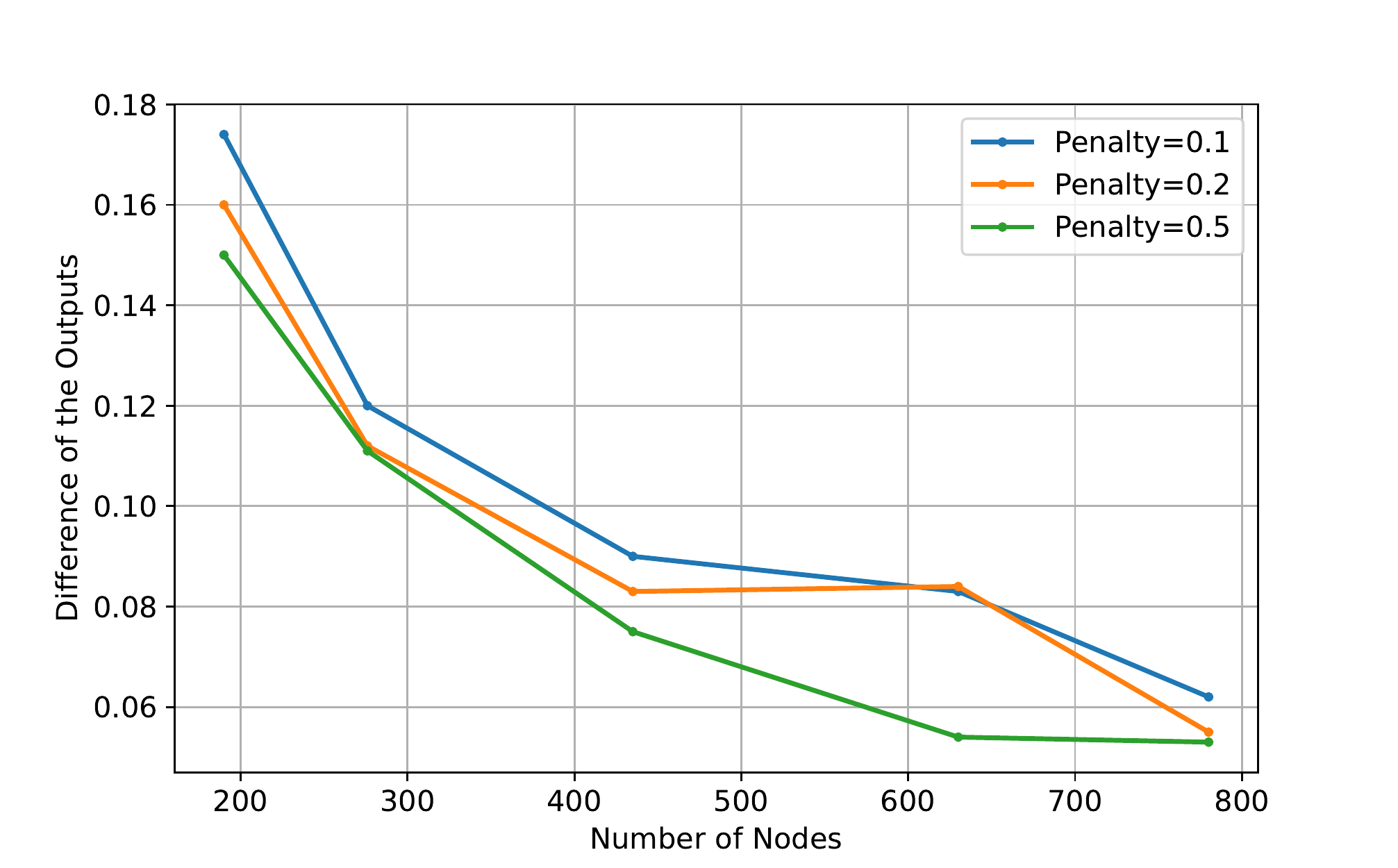}
         \caption{Differences of the outputs of Lipshitz GNN with penalty terms.}
         \label{fig:diff_1layer-penalty}
\end{figure}

\myparagraph{Transferability verification} 
We verify the transferability result presented in Corollary \ref{cor:transferability} by training the architectures on a sampled geometric graph with size $n=435, 630, 780, 1225$ and testing these architectures on the large geometric graphs with $n=1225$. Note that due to the non-scability of the linear readout layer, we keep the graph filter functions unchanged and only retrain the final linear layer to map the outputs to a 2-d direction vector. The successful rates are shown in Table \ref{tab:trans-navi}. We can observe that the architectures trained on smaller sampled geometric graphs can still perform well when implemented on larger graphs. Moreover, Lipschitz GNN performs slightly better than GNN which outperforms Graph Filters. This can be understood as the Lipschitz continuity of the filter functions and the nonlinearity functions help with the transferability as we have discussed in Section \ref{thm:converge-MNN}.
%%%%%%%%%%%%%%%%%%%%%%%%%%%%%%%%%%%%%%%%%%%%%%%%
%%%%%%%%%%%%%%%%%% TABLE %%%%%%%%%%%%%%%%%%%%%%% 
%%%%%%%%%%%%%%%%%%%%%%%%%%%%%%%%%%%%%%%%%%%%%%%%
\begin{table}[h!]
\centering
\begin{tabular}{l|c| c |c} \hline
     & Graph Filter  & GNN &  Lipschitz GNN \\ \hline
$n=435$		&0.74  & 0.74 & 0.76 \\ \hline
$n=630$	& 0.79  &  0.8 & 0.78 \\ \hline
$n=780$		&0.81  & 0.82 & 0.83  \\ \hline
$n=1225$	& 0.82  &  0.83 & 0.84 \\ \hline
\end{tabular}
\caption{Successful rates when testing the architectures trained on geometric graphs with $n=435, 630, 780, 1225$ on the geometric graph with $n=1225$. }
\label{tab:trans-navi}
\vspace{-4mm}
\end{table} 

To make the results more explicit, we show the predicted directions for each point over the space tested with a Lipschitz GNN trained on graph with $n=435$ in Figure \ref{fig:direction}. The red arrows point to the learned directions starting from each unlabeled points while the blue arrows are the learned directions for labeled points.  We can see that Lipschitz GNN can efficiently learn the  successful directions for unlabeled points based on these geometric graphs and can transfer to larger graphs.

% \subsection{Ground robot mismatch prediction}
% We further verify with an experiment based on real-world dataset

% %%%%%%%%%%%%%%%%%%%%%%%%%%%%%%%%%%%%%%%%%%%%%%%%
% %%%%%%%%%%%%%%%%%% TABLE %%%%%%%%%%%%%%%%%%%%%%% 
% %%%%%%%%%%%%%%%%%%%%%%%%%%%%%%%%%%%%%%%%%%%%%%%%
% \begin{table}[h!]
% \centering
% \begin{tabular}{l|c| c} \hline
%     & Pavement  & Grass\\ \hline
% 1-layer graph filter		&  &    \\ \hline
% 2-layer graph filter	& 0.006  &  0.0134 \\ \hline
% 1-layer GNN		&  &  \\ \hline
% 2-layer GNN	&0.005  & 0.011\\ \hline
% \end{tabular}
% \caption{Error rates of the mismatch prediction. }
% \label{tb:results}
% \vspace{-4mm}
% \end{table} 

\subsection{Point Cloud Model Classification}
We further evaluate the convergence results on the ModelNet10 \cite{wu20153d} classification problem. The dataset includes 3,991 meshed CAD models from 10 categories for training and 908 models for testing. In each model, $n$ points are uniformly randomly selected to construct geometric graphs to approximate the underlying model, such as chairs, tables. Figure \ref{fig:pointcloud} shows the point cloud model with different sampling points. Our goal is to identify the models for chairs from other models.
\begin{figure}[h!]
\centering
\includegraphics[width=0.12\textwidth]{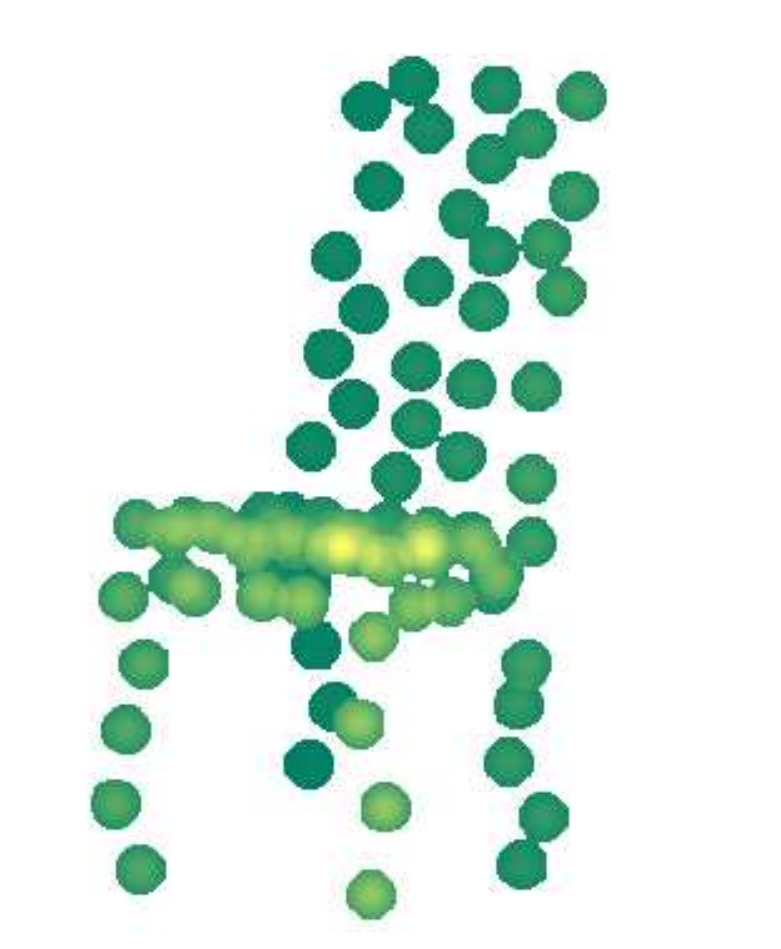}\hspace{-0.3cm}
\includegraphics[width=0.12\textwidth]{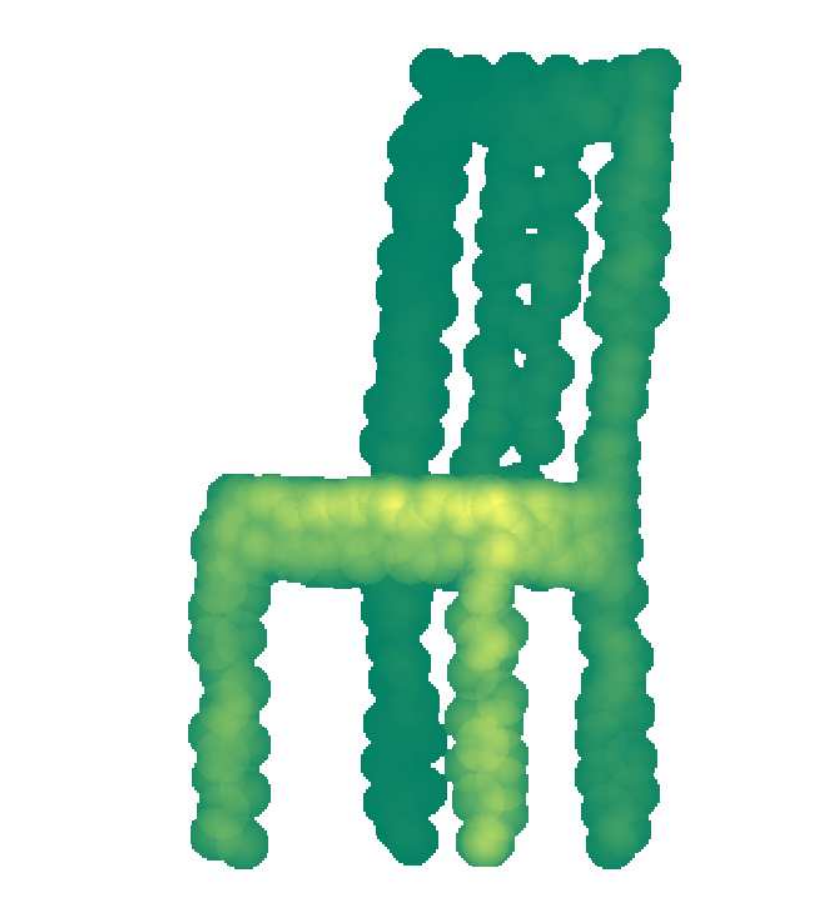} \hspace{0.2cm}
\includegraphics[width=0.127\textwidth]{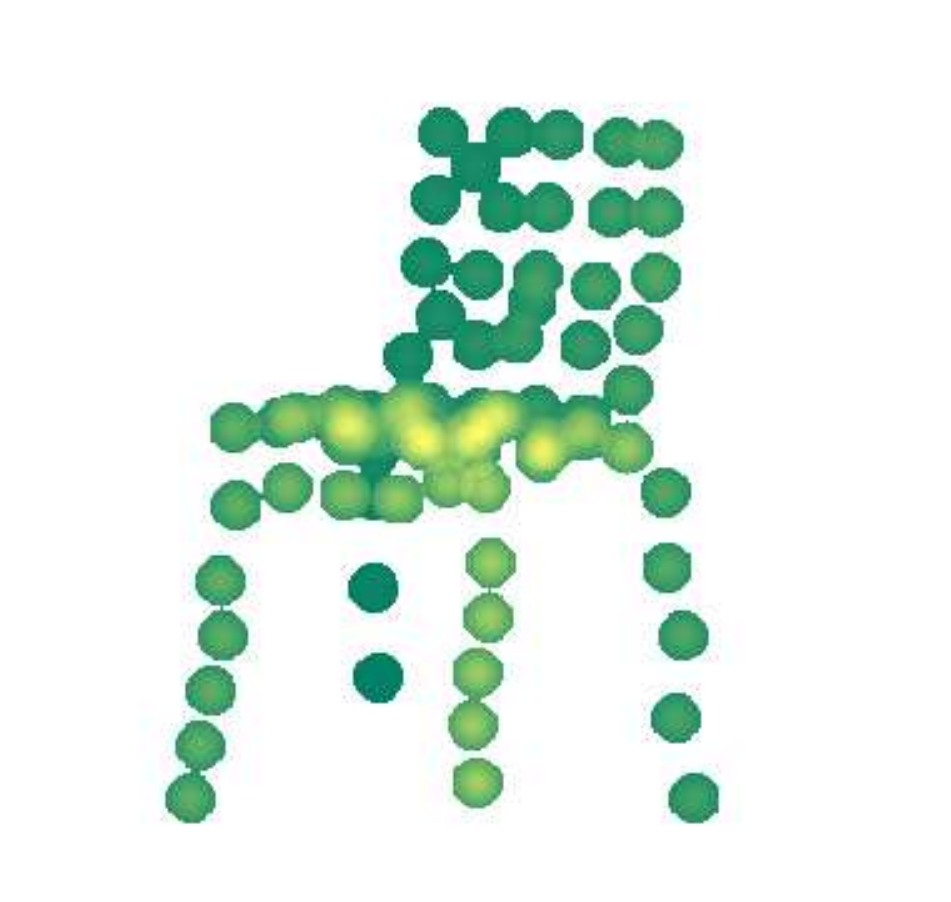}\hspace{-0.3cm}
\includegraphics[width=0.11\textwidth]{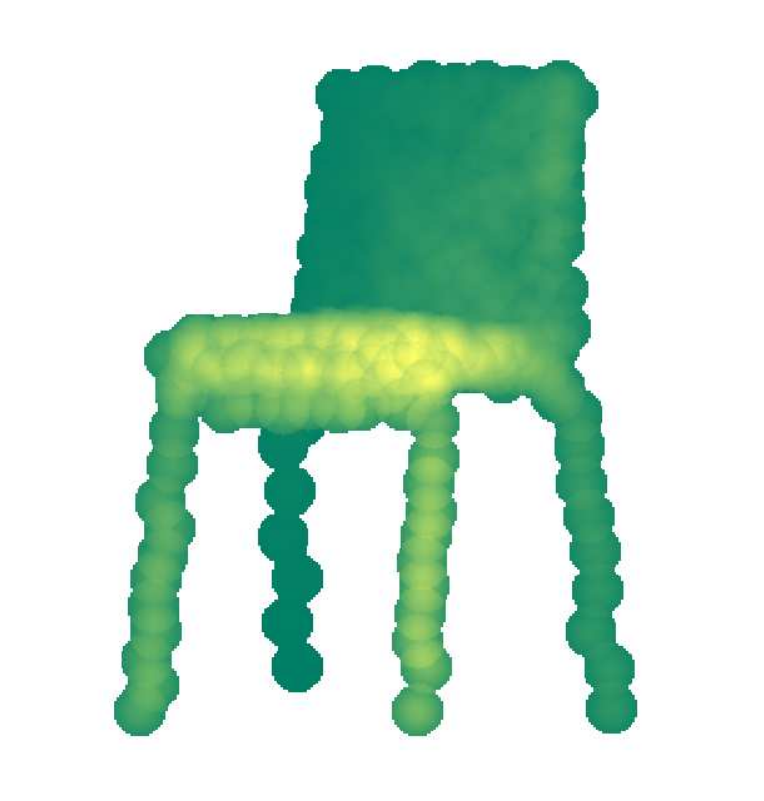}
% \includegraphics[trim=130 0 40 0,clip,width=0.1\textwidth]{chair_sparse.eps} \hspace{-0.5cm}
% \includegraphics[trim=70 0 50 0,width=0.12\textwidth]{chair_large.eps} 
% \hspace{-0.05cm}
% \includegraphics[trim=60 0 100 0,width=0.12\textwidth]{chair_2_sparse.eps}  
% \hspace{-0.05cm}
% \includegraphics[trim=60 0 100 0,width=0.12\textwidth]{chair_large2.eps} 
\caption{Point cloud models with 50 and 300 sampling points in each model. Our goal is to identify chair models from other models such as toilet and table. }
\label{fig:pointcloud}
\end{figure}

\myparagraph{Learning architectures and experiment settings}
We build both dense and sparse geometric graphs as the approximations to the models. We set the coordinates of each point as input signals on the graphs. The weights of the edges are calculated based on the Euclidean distance between the nodes. For the dense geometric graphs, the weight function is determined as \eqref{eqn:gauss_kernel} with $\epsilon = 0.1$. Similarly, the weight function of a sparse geometric graph is calculated as \eqref{eqn:compact_kernel} with $\epsilon = 0.001$ as the threshold. We calculate the Laplacian matrix for each graph as the input graph shift operator. In this experiment, we implement and compare three different architectures, including 2-layer Graph Filters (GF), 2-layer Graph Neural Network (GNN) and 2-layer Lipschitz Graph Neural Network (Lipschitz GNN). The architectures contain $F_0 = 3$ input features which are the 3-d coordinates of each point, $F_1= 64$ and $F_2=32$ features with $K_t=5$ filters in each layer. In GNN and Lipschitz GNN, ReLU is used as the nonlinearity function. The filters in Lipschitz GNN is regularized as Lipschitz continuous by imposing a penalty term $C_L h'(\lambda)$ to the loss function with $C_L$ set as 0.3. All architectures include a linear readout layer to map the final classification outputs. 

All the architectures are trained by minimizing the cross-entropy loss. We implement an ADAM optimizer with the learning rate set as 0.005 along with the forgetting factors 0.9 and 0.999. We carry out the training for 40 epochs with the size of batches set as 10. We run 5 random dataset partitions and show the average estimation error rates and the standard deviation across these partitions.

\myparagraph{Convergence Verification}
We verify our convergence results by training the graph filters and GNNs on geometric graphs with $n= 300, 400, 500, 600, 700,800,900$ sampling points, and plotting the differences between the output of the trained graph filters or GNNs on geometric graphs with size $n$ and on the relatively large enough geometric graphs with size $1000$. Figure \ref{fig:diff_1layer} shows the convergence results for dense graphs and figure \ref{fig:diff_1layer-sparse} shows the results for sparse graphs.
\begin{figure}[h!]
         \centering
       \includegraphics[width=0.47\textwidth]{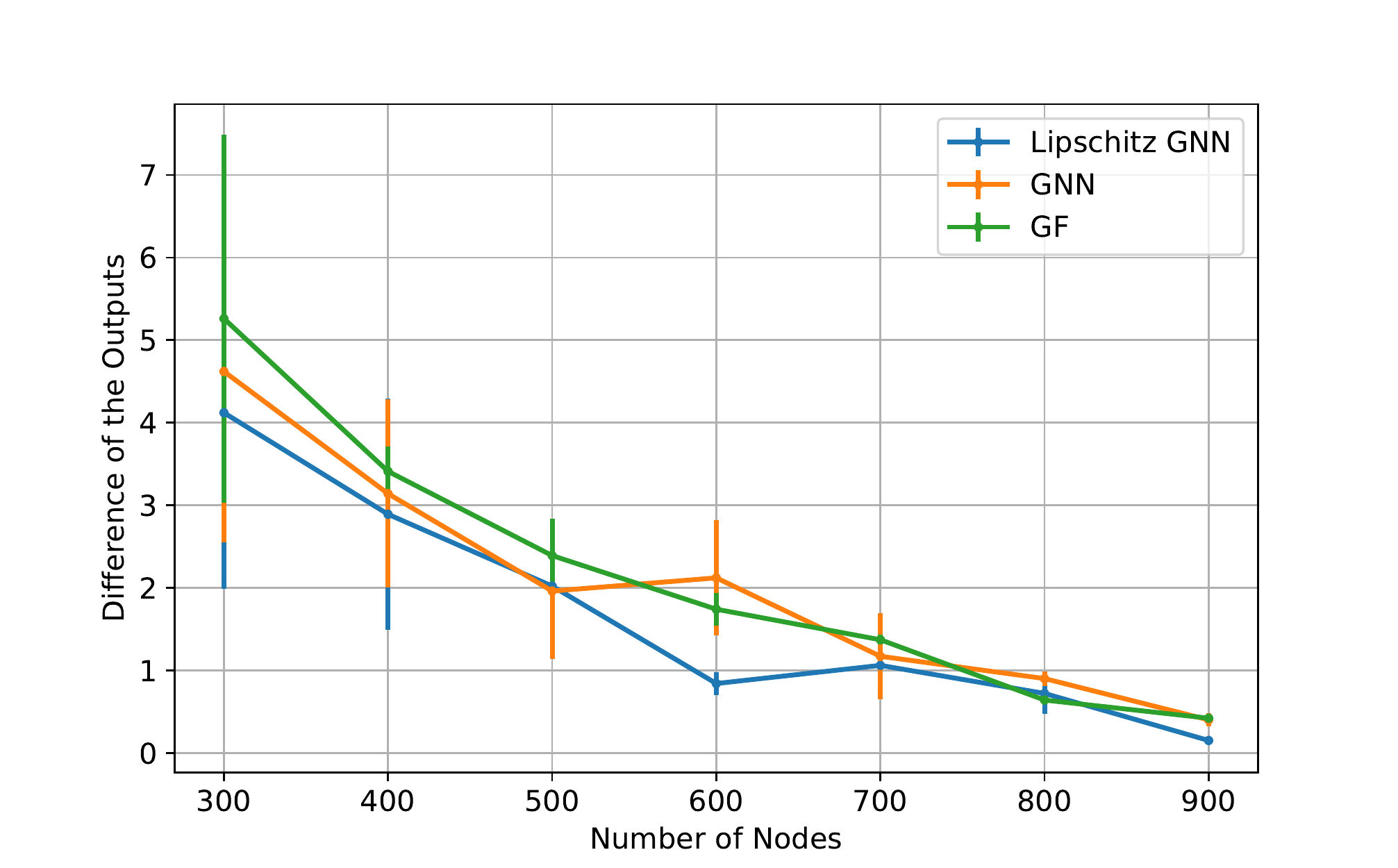}
         \caption{Differences of the outputs of Lipschitz GNN, GNN and graph filter  trained on dense geometric graphs.}
         \label{fig:diff_1layer}
\end{figure}
\begin{figure}[h!]
         \centering
       \includegraphics[width=0.47\textwidth]{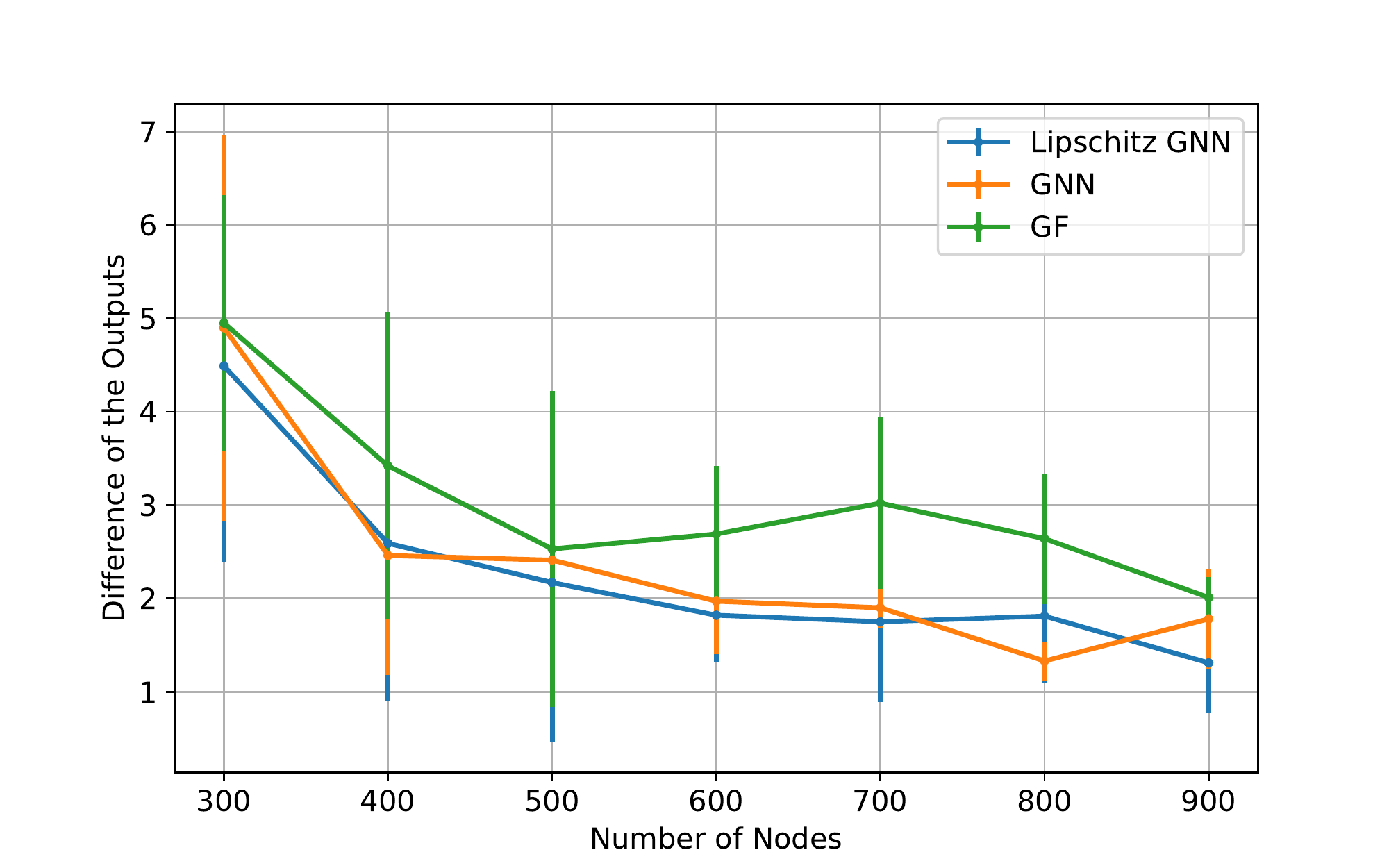}
         \caption{Differences of the outputs of Lipschitz GNN, GNN and graph filter  trained on sparse geometric graphs.}
         \label{fig:diff_1layer-sparse}
\end{figure}
From the figures we can see that the differences between the outputs on the trained smaller geometric graphs and the relatively large geometric graphs decrease and converge as the trained graph size grows, which verifies Corollary \ref{cor:transferability}. This also proves what we state in Theorem \ref{thm:converge-MNN} that the approximation errors of geometric GNNs to the MNNs decrease as the number of nodes grows if we see the large enough graph with $n=1000$ nodes as the manifold. We can also observe that Lipschitz GNNs have better approximations than GNNs because of the continuous filter function while GNNs outperform graph filters due to the nonlinearity function.

\myparagraph{Transferability verification}
 We train the architectures on both dense and sparse geometric graphs  with $n=300, 500, 700, 900$ sampled points from each CAD model. To justify the theoretical results that we have stated in Section \ref{sec:converge_gnn}, we test these trained architectures on a relatively large graph containing $n=1000$ sampled points to verify that geometric GNNs have the transferability to large graphs, i.e. the trained geometric graph filters and GNNs can be directly implemented on larger geometric graphs as long as the geometric graphs are constructed in the same manner. The classification error rates are shown in Table \ref{tab:pointcloud} for the dense graph setting and in Table \ref{tab:pointcloud-sparse} for the sparse graph setting.
 %%%%%%%%%%%%%%%%%%%%%%%%%%%%%%%%%%%%%%%%%%%%%%%%
%%%%%%%%%%%%%%%%%% TABLE %%%%%%%%%%%%%%%%%%%%%%% 
%%%%%%%%%%%%%%%%%%%%%%%%%%%%%%%%%%%%%%%%%%%%%%%%
\begin{table}[h!]
\centering
\begin{tabular}{l|c|c|c} \hline
 	& Graph Filters & GNN & Lipschitz GNN   \\ \hline
$n=300$	&  $21.15\% \pm 3.48\%$  & $9.35\% \pm 2.46\%$ &  $7.63\% \pm3.36\%$ \\ \hline
$n=500$		&  $18.09\% \pm 6.28\%$ &   $7.80\% \pm 3.50\%$ &  $7.54\% \pm 4.01\%$ \\ \hline
$n=700$	&  $17.31\% \pm 6.59\%$ &  $8.16\% \pm 2.95\%$ &  $7.97\% \pm 2.45\%$ \\ \hline
$n=900$ & $15.58\% \pm 4.54\%$ &  $7.20\% \pm 3.77\%$ &  $6.68\% \pm 3.94\%$ \\ \hline
\end{tabular}
\caption{Classification error rates for model 'chair' when testing the architectures trained on dense geometric graphs with $n=300, 500, 700, 900$ to dense geometric graphs with $n=1000$. Average over 5 data realizations.}
\label{tab:pointcloud}
\vspace{-3mm}
\end{table} 

 %%%%%%%%%%%%%%%%%%%%%%%%%%%%%%%%%%%%%%%%%%%%%%%%
%%%%%%%%%%%%%%%%%% TABLE %%%%%%%%%%%%%%%%%%%%%%% 
%%%%%%%%%%%%%%%%%%%%%%%%%%%%%%%%%%%%%%%%%%%%%%%%
\begin{table}[h!]
\centering
\begin{tabular}{l|c|c|c} \hline
 	& Graph Filters & GNN & Lipschitz GNN   \\ \hline
$n=300$	&  $19.83\% \pm 5.94\%$  & $7.74\% \pm 4.05\%$ &  $ 7.68\% \pm 3.75\%$ \\ \hline
$n=500$		&  $21.97\% \pm 4.17\%$ &   $10.10\% \pm 1.40\%$ &  $8.60\% \pm 2.95\%$ \\ \hline
$n=700$	&  $13.85\% \pm 3.81\%$ &  $7.45\% \pm 4.03\%$ &  $8.02\% \pm 2.77\%$ \\ \hline
$n=900$ & $16.62\% \pm 2.38\%$ &  $7.92\% \pm 3.14\%$ &  $7.44\% \pm 3.30\%$ \\ \hline
\end{tabular}
\caption{Classification error rates for model 'chair' when testing the architectures trained on sparse geometric graphs with $n=300, 500, 700, 900$ to sparse geometric graphs with $n=1000$. Average over 5 data realizations.}
\label{tab:pointcloud-sparse}
\vspace{-3mm}
\end{table} 
We can see from the results shown in the tables that the trained architectures can still perform well on the relatively large geometric graphs, both for dense and sparse graph settings. We can see that Lipschitz GNNs outperform GNNs, which perform better than graph filters. This attests the effects of the filter function continuity and nonlinearity that we have discussed in Section \ref{sec:converge_gnn}, which is continuous filter function and nonlinearity function both help with improving transferability. We can also observe that architectures trained on geometric graphs with more number of nodes can achieve better performances on the relatively large graphs. This can be understood when we see this large enough geometric graph with $n=1000$ approximately as the underlying manifold and as Theorem \ref{thm:converge-MNN} shows the geometric GNNs can give better approximations to the MNNs as the number of nodes grow. This can also be understood as the transferability property presented in Corollary \ref{cor:transferability}, which also indicates that graph filters and GNNs trained on larger geometric graphs have better performance approximations. Furthermore, the results also show that architectures trained on dense geometric graphs transfer better than the ones trained on sparse geometric graphs, which is also accordant with what we have claimed in Section \ref{sec:converge_gnn}.

\section{Conclusions}
\label{sec:conclusion}
%!TEX root = stability_manifold_TSP.tex
In this paper, we import the definition of manifold convolutional filters with an exponential Laplace-Beltrami operator to process manifold signals. The manifold model is accessible with a set of i.i.d. uniformly sampled points over the manifold. We construct both dense and sparse graph models to approximate the underlying manifold. We first prove the approximation error bounds of discrete graph Laplacians to the LB operator in the spectral domain. We transfer the definition of manifold filter to the constructed graph and prove the graph filter can approximate the manifold filter with a non-asymptotic error bound.  The constructed graph filters need to trade-off between the discrimative and approximative powers while GNNs composed with graph filters and nonlinearities can alleviate the trade-off with the frequency mixing by nonlinearities. We conclude that the GNNs are thus both good approximations of MNNs and discriminative.
We finally verified our results numerically with a navigation control problem over a manifold.

%%%%%%%%%%%%%%%%%%%%%%%%%%%%%%%%%%%%%%%%%%%%%%%%%%%%%%%%%%%%

\urlstyle{same}
\bibliographystyle{IEEEtran}
\bibliography{references}

\appendix
 {\section{Appendix}

\subsection{Proof of Theorem \ref{thm:converge-MF-dense} and Theorem \ref{thm:converge-MF-sparse}}
\label{app:nn}
 
We first write out the filter representation as
 \begin{align}
    &\nonumber \|\bbh(\bbL_n^\epsilon)\bbP_n f - \bbP_n\bbh(\ccalL) f\|\\
    &\leq \left\| \sum_{i=1}^\infty \hat{h}(\lambda_{i,n}^\epsilon) \langle \bbP_nf,\bm\phi_{i,n}^\epsilon \rangle_{\bbG_n}\bm\phi_{i,n}^\epsilon - \sum_{i=1}^\infty \hat{h}(\lambda_i)\langle f,\bm\phi_i\rangle_{\ccalM} \bbP_n \bm\phi_i  \right\|
    %  \\ 
    %  &\nonumber \leq  \left\| \sum_{i=1}^M \hat{h}(\lambda_{i,n}^\epsilon) \langle \bbP_nf,\bm\phi_{i,n}^\epsilon \rangle_{\bbG_n}\bm\phi_{i,n}^\epsilon - \sum_{i=1}^M \hat{h}(\lambda_i) \langle \bbP_nf,\bm\phi_{i,n}^\epsilon \rangle_{\bbG_n}\bm\phi_{i,n}^\epsilon\right\| \\
    %  & \qquad \qquad \qquad \qquad\qquad \qquad+\left\| \sum_{i=1}^M \hat{h}(\lambda_i) \langle \bbP_n f,\bm\phi_{i,n}^\epsilon \rangle_{\bbG_n} \bm\phi_{i,n}^\epsilon - \sum_{i=1}^M \hat{h}(\lambda_i) \langle f,\bm\phi_i \rangle_{\ccalM} \bbP_n \bm\phi_i \right\|.\label{eqn:conv-1}
 \end{align}
 
 We  denote the index of partitions that contain a single eigenvalue
as a set $\ccalK_s$ ($|\ccalK_s|= N_s$) and the rest as a set $\ccalK_m$ ($|\ccalK_m|= N_m$). We decompose the $\alpha$-FDT filter function as $\hat{h}(\lambda)=h^{(0)}(\lambda)+\sum_{l\in\ccalK_m}h^{(l)}(\lambda)$ as
 \begin{align}
\label{eqn:h0}& h^{(0)}(\lambda) = \left\{ 
\begin{array}{cc} 
              \hat{h}(\lambda)-\sum\limits_{l\in\ccalK_m}\hat{h}(C_l)  &  \lambda\in[\Lambda_k(\alpha)]_{k\in\ccalK_s} \\
                0& \text{otherwise}  \\
                \end{array} \right. \\
\label{eqn:hl}& h^{(l)}(\lambda) = \left\{ 
\begin{array}{cc} 
                \hat{h}(C_l) &  \lambda\in[\Lambda_k(\alpha)]_{k\in\ccalK_s} \\
                \hat{h}(\lambda) & 
                \lambda\in\Lambda_l(\alpha)\\
                0 &
                \text{otherwise}  \\
                \end{array} \right.             
\end{align}
 with $C_l$ some constant in $\Lambda_l(\alpha)$.
 With the triangle inequality, we start by analyzing the output difference of $h^{(0)}(\lambda)$ as
 \begin{align}
    & \nonumber \| \sum_{i=1}^\infty {h}^{(0)}(\lambda_{i,n}^\epsilon) \langle \bbP_nf,\bm\phi_{i,n}^\epsilon \rangle_{\bbG_n}\bm\phi_{i,n}^\epsilon - \sum_{i=1}^\infty {h}^{(0)}(\lambda_i)\langle f,\bm\phi_i\rangle_{\ccalM} \bbP_n \bm\phi_i  \|
     \\ 
     &\nonumber \leq  \left\| \sum_{i=1}^\infty \left({h}^{(0)}(\lambda_{i,n}^\epsilon)- {h}^{(0)}(\lambda_i) \right) \langle \bbP_nf,\bm\phi_{i,n}^\epsilon \rangle_{\bbG_n}\bm\phi_{i,n}^\epsilon \right\| \\
     &  +\left\| \sum_{i=1}^\infty {h}^{(0)}(\lambda_i)\left( \langle \bbP_n f,\bm\phi_{i,n}^\epsilon \rangle_{\bbG_n} \bm\phi_{i,n}^\epsilon - \langle f,\bm\phi_i \rangle_{\ccalM} \bbP_n \bm\phi_i \right)  \right\|.\label{eqn:conv-1}
 \end{align}
 
 The first term in \eqref{eqn:conv-1} can be bounded by leveraging the $A_h$-Lipschitz continuity of the frequency response. From the eigenvalue difference in Proposition \ref{thm:converge-spectrum-sparse}, we can claim that for each eigenvalue $\lambda_i \leq \lambda_K$, we have
\begin{gather}
 \label{eqn:eigenvalue}  |\lambda_{i,n}^\epsilon-\lambda_i|\leq \Omega_{1,K}\sqrt{\epsilon}.
 \end{gather}
The square of the first term is bounded as 
\begin{align}
   &\nonumber \left\| \sum_{i=1}^\infty ({h}^{(0)}(\lambda_{i,n}^\epsilon) - {h}^{(0)}(\lambda_i)) \langle \bbP_n f,\bm\phi_{i,n}^\epsilon \rangle_{\bbG_n} \bm\phi_{i,n}^\epsilon  \right\|^2 \\
   & \leq \sum_{i=1}^\infty |{h}^{(0)}(\lambda_{i,n}^\epsilon)-{h}^{(0)}(\lambda_i)|^2 |\langle \bbP_n f,\bm\phi_{i,n}^\epsilon \rangle_{\bbG_n}|^2 \\
   &\leq \sum_{i=1}^\infty A_h^2 |\lambda_{i,n}^\epsilon-\lambda_i| ^2\|\bbP_n f\|^2\leq A_h^2 \Omega_{1,K}^2  \epsilon.
\end{align}

The second term in \eqref{eqn:conv-1} can be bounded combined with the convergence of eigenfunctions in \eqref{eqn:eigenfunction} as
\begin{align}
  & \nonumber \Bigg\| \sum_{i=1}^\infty {h}^{(0)}(\lambda_i)\left( \langle \bbP_nf,\bm\phi_{i,n}^\epsilon \rangle_{\bbG_n}\bm\phi_{i,n}^\epsilon - \langle f,\bm\phi_i \rangle_{\ccalM} \bbP_n \bm\phi_i\right)  \Bigg\|\\
   & \leq \nonumber \Bigg\|  \sum_{i=1}^\infty {h}^{(0)}(\lambda_i)  \left(\langle \bbP_n f,\bm\phi_{i,n}^\epsilon\rangle_{\bbG_n}\bm\phi_{i,n}^\epsilon  - \langle \bbP_nf,\bm\phi_{i,n}^\epsilon \rangle_{\bbG_n} \bbP_n\bm\phi_i\right)\Bigg\|\\
   &\label{eqn:term1}+ \left\| \sum_{i=1}^\infty  {h}^{(0)}(\lambda_i) \left(\langle \bbP_n f,\bm\phi_{i,n}^\epsilon\rangle_{\bbG_n} \bbP_n\bm\phi_i -\langle f,\bm\phi_i\rangle_\ccalM \bbP_n\bm\phi_i \right) \right\|
%   &\leq \nonumber \sum_{i=1}^n \langle \bbP_nf,\bm\phi_{i,n}^\epsilon \rangle \| \bm\phi_{i,n}^\epsilon-\bbP_n\bm\phi_i \|\\
%   &\qquad \qquad + \sum_{i=1}^n|\langle \bbP_n f, \bm\phi_{i,n}^\epsilon\rangle -\langle f,\bm\phi_i\rangle|\left\| \bbP_n\bm\phi_i \right\|
\end{align}
From the convergence stated in Theorem \ref{thm:converge-spectrum-sparse}, we have
\begin{gather}
 \label{eqn:eigenfunction}    \|a_i \bm\phi_{i,n}^\epsilon-\bm\phi_i\|\leq \Omega_{2,K} \sqrt{\epsilon}/ \theta,
 \end{gather}
 with the eigengap $\theta \geq \alpha$ under the $\alpha$-FDT filter. Therefore, the first term in \eqref{eqn:term1} can be bounded as
\begin{align}
& \nonumber \left\|  \sum_{i=1}^\infty {h}^{(0)}(\lambda_i) \left(\langle \bbP_n f,\bm\phi_{i,n}^\epsilon\rangle_{\bbG_n}\bm\phi_{i,n}^\epsilon  - \langle \bbP_nf,\bm\phi_{i,n}^\epsilon \rangle_{\ccalM} \bbP_n\bm\phi_i\right)\right\|\\
& \qquad \qquad\leq \sum_{i=1}^{N_s} \|\bbP_n f\|\|\bm\phi_{i,n}^\epsilon - \bbP_n\bm\phi_i\|\leq \frac{N_s \Omega_{2,K}}{\alpha}\sqrt{\epsilon}.
\end{align}
The last equation comes from the definition of norm in $L^2(\bbG_n)$.
The second term in \eqref{eqn:term1} can be written as
\begin{align}
     & \nonumber \Bigg\| \sum_{i=1}^\infty {h}^{(0)}(\lambda_{i,n}^\epsilon) (\langle \bbP_n f,\bm\phi_{i,n}^\epsilon\rangle_{\bbG_n}  \bbP_n\bm\phi_i -\langle f,\bm\phi_i\rangle_\ccalM \bbP_n\bm\phi_i ) \Bigg\| \\
   &\leq \sum_{i=1}^\infty |{h}^{(0)}(\lambda_{i,n}^\epsilon)| \left|\langle \bbP_n f,\bm\phi_{i,n}^\epsilon\rangle_{\bbG_n}  -\langle f,\bm\phi_i\rangle_\ccalM\right|\|\bbP_n\bm\phi_i\|.
\end{align}
Because $\{x_1, x_2,\cdots,x_n\}$ is a set of uniform sampled points from $\ccalM$, based on Theorem 19 in \cite{von2008consistency} we can claim that
\begin{equation}
   \left|\langle \bbP_n f,\bm\phi_{i,n}^\epsilon\rangle_{\bbG_n}  -\langle f,\bm\phi_i\rangle_\ccalM\right| = O\left(\sqrt{\frac{\log n}{n}}\right).
\end{equation}
Taking into consider the boundedness of frequency response $|{h}^{(0)}(\lambda)|\leq 1$ and the bounded energy $\|\bbP_n\bm\phi_i\|$. Therefore, we have 
\begin{align}
&\nonumber  \left\| \sum_{i=1}^\infty \hat{h}(\lambda_{i,n}^\epsilon) \left(\langle \bbP_n f,\bm\phi_{i,n}^\epsilon\rangle_{\bbG_n}  -\langle f,\bm\phi_i\rangle_\ccalM \right)\bbP_n\bm\phi_i  \right\|= O\left(\sqrt{\frac{\log n}{n}}\right).
\end{align}

Combining the above results, we can bound the output difference of $h^{(0)}$. Then we need to analyze the output difference of $h^{(l)}(\lambda)$ and bound this as
\begin{align}
    \nonumber &\left\| \bbP_n \bbh^{(l)}(\ccalL)f -\bbh^{(l)}(\bbL_n^\epsilon)\bbP_n f \right\| 
    \\& \leq \left\| (\hat{h}(C_l)+\gamma)\bbP_n f - (\hat{h}(C_l)-\gamma)\bbP_nf\right\| \leq 2\gamma\|\bbP_nf\|,
\end{align}
where $\bbh^{(l)}(\ccalL)$ and $\bbh^{(l)}(\bbL_n^\epsilon)$ are filters with filter function $h^{(l)}(\lambda)$ on the LB operator $\ccalL$ and graph Laplacian $\bbL_n^\epsilon$ respectively.
Combining the filter functions, we can write
\begin{align}
   \nonumber &\|\bbP_n\bbh(\ccalL)f-\bbh(\bbL_n^\epsilon)\bbP_n f\|\\\nonumber &=
    \Bigg\|\bbP_n\bbh^{(0)}(\ccalL)f +\bbP_n\sum_{l\in\ccalK_m}\bbh^{(l)}(\ccalL)f -\\& \qquad \qquad \qquad \bbh^{(0)}(\bbL_n^\epsilon)\bbP_n f - \sum_{l\in\ccalK_m} \bbh^{(l)}(\bbL_n^\epsilon)\bbP f \Bigg\|\\
    &\nonumber \leq \|\bbP_n \bbh^{(0)}(\ccalL)f-\bbh^{(0)}(\bbL_n^\epsilon)\bbP_n f\|+\\
    &\qquad \qquad \qquad \sum_{l\in\ccalK_m}\|\bbP_n \bbh^{(l)}(\ccalL)f-\bbh^{(l)}(\bbL_n^\epsilon)\bbP_nf\|\\
    &\nonumber \leq A_h \Omega_{1,K}\sqrt{\epsilon}+\frac{N_s\Omega_{2,K}}{\alpha} \sqrt{\epsilon} +N_m \gamma +C_{gc}\sqrt{\frac{\log(n)}{n}}
\end{align}
With $\gamma = \Omega_{2,K}\sqrt{\epsilon}/\alpha$, we have
\begin{align}
   &\nonumber  \|\bbh(\bbL_n^\epsilon)\bbP_n f - \bbP_n\bbh(\ccalL) f\|\\
   &\qquad  \leq \left( \frac{N\Omega_{2,K}}{\alpha} +A_h \Omega_{1,K}\right)\sqrt{\epsilon} + C_{gc}\sqrt{\frac{\log n}{n}}
\end{align}

We can prove Theorem \ref{thm:converge-MF-sparse} similarly by importing Proposition \ref{thm:converge-spectrum-sparse} into the eigenvalue and eigenfunction differences.}

\clearpage
\setcounter{page}{1}
\begin{center}
\textbf{\large Supplemental Materials}
\end{center}

\section{Supplementary Materials}
%%%%%%%%%%%%%%%%%%%%%%%%%%%%%%%%%%%%%%%%%%%%%%%%
%%%%%%%%%%%%%%%%%% SUBSECTION %%%%%%%%%%%%%%%%%% 
%%%%%%%%%%%%%%%%%%%%%%%%%%%%%%%%%%%%%%%%%%%%%%%%
\setcounter{subsection}{0}
\subsection{Proof of Proposition \ref{thm:operator-diff}}
\label{app:operator}
We decompose the operator difference between the graph Laplacian and the LB operator with an intermediate term $\bbL^\epsilon$, which is the functional approximation defined in \eqref{eqn:functional_laplacian}. We first focus on the operator difference between $\bbL^\epsilon$ and $\ccalL$. From \cite{belkin2008towards}, we can get the bound as
\begin{equation}
    \|\bbL^\epsilon \bm\phi_i- \ccalL \bm\phi_i\|\leq C\sqrt{\epsilon} \|\bm\phi_i\|_{H^{d/2+1}},
\end{equation}
For the Sobolev norm of eigenfunction $\bm\phi_i$, according to \cite[Lemma~4.4]{belkin2006convergence} we have
\begin{equation}
    \|\bm\phi_i\|_{H^{d/2+1}}\leq C \lambda_i^{\frac{d+2}{4}},
\end{equation}
which leads to 
\begin{equation}
\label{eqn:Lepsilon-L}
    \|\bbL^\epsilon \bm\phi_i- \ccalL \bm\phi_i\|\leq C_1 \sqrt{\epsilon} \lambda_i^{\frac{d+2}{4}}.
\end{equation}
For the operator difference between $\bbL_n^\epsilon$ and $\bbL^\epsilon$ with Hoeffding's inequality as 
\begin{align}
    \mathbb{P}\left( |\bbL_n^\epsilon \bm\phi_i(x) - \bbL^\epsilon \bm\phi_i(x)|> \epsilon_1\right) \leq \exp\left( - \frac{2n\epsilon_1^2}{\|\bm\phi_i\|_{H^{d/2+1}}^2} \right).
\end{align}
Therefore, we can claim that with probability at least $1-\delta$, we have
\begin{align}
\label{eqn:Lnepsilon-Lepsilon}
    |\bbL_n^\epsilon \bm\phi_i(x) - \bbL^\epsilon \bm\phi_i(x)|\leq \sqrt {\frac{\ln 1/\delta}{2n}} \|\bm\phi_i\|_{H^{d/2+1}}.
\end{align}
Combining \eqref{eqn:Lepsilon-L} and \eqref{eqn:Lnepsilon-Lepsilon} with triangle inequality, we can get the conclusion in Theorem \ref{thm:operator-diff}.

\subsection{Proof of Proposition \ref{thm:converge-spectrum}}
We first import two lemmas to help prove the spectral properties.
\label{app:spectrum}
\begin{lemma}\label{lem:conv_eigenfunction}
Let $\bbA, \bbB$ be self-adjoint operators with $\{\lambda_i(\bbA), \bbu_i\}_{i=1}^\infty$ and $\{\lambda_i(\bbB), \bbw_i\}_{i=1}^\infty$ as the corresponding spectrum. Let $Pr_{\bbw_i}$ be the orthogonal projection operation onto the subspace generated by $\bbw_i$. Then we have
\begin{equation}
    \|a_i\bbu_i-\bbw_i\|\leq 2\|\bbu_i-Pr_{\bbw_i}\bbu_i\|\leq \frac{2\|\bbB\bbu_i-\bbA\bbu_i\|}{\min_{j\neq i} |\lambda_j(\bbB)-\lambda_i(\bbA)|}.
\end{equation}
\end{lemma}
\begin{proof}
The first inequality is directly from \cite[Proposition~18]{von2008consistency}. Let $Pr^\perp_{\bbw_i}$ be the orthogonal projection onto the complement of the subspace generated by $\bbw_i$. Then we have
\begin{align}
    \|\bbu_i - Pr_{\bbw_i} \bbu_i\|=\|Pr_{\bbw_i}^\perp \bbu_i\|=\Big\|\sum_{j\neq i}\langle \bbu_i, \bbw_j \rangle \bbw_j\Big\|.
\end{align}
Therefore, we have
\begin{align}
  \nonumber  &\|Pr_{\bbw_i}^\perp \bbB\bbu_i - Pr_{\bbw_i}^\perp \bbA\bbu_i\| \\
    & = \Big\| \sum_{j\neq i}\langle \bbB\bbu_i,\bbw_j \rangle \bbw_j -\sum_{j\neq i}\langle \bbA\bbu_i,\bbw_j \rangle\bbw_j \Big\|\\
    &=\Big\| \langle \bbu_i, \bbB\bbw_j \rangle\bbw_j -\sum_{j\neq i} \lambda_i(\bbA)\langle\bbu_i, \bbw_j \rangle \bbw_j\Big\|\\
    &=\Big\| \sum_{j\neq i}(\lambda_i(\bbB)-\lambda_i(\bbA)) \langle \bbu_i,\bbw_j \rangle\bbw_j\Big\|\\
    &\geq \min_{j\neq i} | \lambda_i(\bbB)-\lambda_i(\bbA) | \|\sum_{j\neq i}\langle \bbu_i,\bbw_j \rangle\bbw_j \| \\
    &= \min_{j\neq i} | \lambda_i(\bbB)-\lambda_i(\bbA) | \| \bbu_i -Pr_{\bbw_i}\bbu_i \|,
\end{align}
together with $\|\bbB\bbu_i - \bbA\bbu_i \| \geq  \|Pr_{\bbw_i}^\perp \bbB\bbu_i - Pr_{\bbw_i}^\perp \bbA\bbu_i\|$. We can conclude the proof.
\end{proof}
The following lemma is adapted from \cite[Lemma~5c]{dunson2021spectral}
\begin{lemma}\label{lem:conv_eigenvalue}
Let $\bbA, \bbB$ be self-adjoint operators with $\{\lambda_i(\bbA), \bbu_i\}_{i=1}^\infty$ and $\{\lambda_i(\bbB), \bbw_i\}_{i=1}^\infty$ as the corresponding spectrum. Then we have
\begin{equation}
    |\lambda_i(\bbA)-\lambda_i(\bbB)|=\frac{\langle(\bbA-\bbB)\bbu_i,\bbw_i\rangle}{|\langle \bbu_i,\bbw_i \rangle|}\leq \frac{\|(\bbA-\bbB)\bbu_i\|}{|\langle \bbu_i,\bbw_i \rangle|}
\end{equation}
\end{lemma}

With the above lemmas and our proposed Theorem \ref{thm:operator-diff}, which includes the operator difference, we can prove Theorem \ref{thm:converge-spectrum}. We first fix some $K\in\mathbb{N}$, wich provides an upper bound for $\lambda_i\leq \lambda_K$ for all $1\leq i\leq K$. By taking the probability $1-n^{-2}$ and $\epsilon = n^{-1/(d+4)}$, we can conclude that the operator difference in Theorem \ref{thm:operator-diff} can be bounded with order $O(\sqrt{\epsilon})$, with the constant scaling with $\lambda_K^{\frac{d+2}{4}}$. Combine with Lemma \ref{lem:conv_eigenfunction} and $\theta =\min_{1\leq j\neq i \leq K} |\lambda_j-\lambda_{i,n}^\epsilon|$, we can get 
\begin{equation}
    \|a_i\bm\phi_{i,n}^\epsilon - \bm\phi_i\|\leq \frac{C_k}{\theta}\sqrt{\epsilon},
\end{equation}
where we denote the constant as $\Omega_1$ to include the effects of $K$, the eigengap and the volume of $\ccalM$.

This upper bound of the eigenfunction difference leads to $|\langle\bbu_i,\bbw_i \rangle |\geq 1-\Omega_1/2\sqrt{\epsilon}\geq 1$. Combining with Lemma \ref{lem:conv_eigenvalue}, the difference of the eigenvalues can also be bounded in the order of $O(\sqrt{\epsilon})$.

\subsection{Proof of Theorem \ref{thm:converge-MNN}}
\label{app:mnn}
To bound the output difference of MNNs, we need to write in the form of features of the final layer
 \begin{align}
    \nonumber \|\bm\Phi(\bbH,\bbL_n^\epsilon,\bbP_nf)-\bbP_n \bm\Phi&(\bbH,\ccalL, f))\| = \left\| \sum_{q=1}^{F_L}\bbx_{n,L}^q-\sum_{q=1}^{F_L}\bbP_n f_L^q \right\|\\
     & \leq \sum_{q=1}^{F_L} \left\| \bbx_{n,L}^q- \bbP_n f_L^q \right\|.
 \end{align}
By inserting the definitions, we have 
 \begin{align}
   \nonumber  &\left\| \bbx_{n,l}^p- \bbP_n f_l^p \right\|\\
     &=\left\| \sigma\left(\sum_{q=1}^{F_{l-1}} \bbh_l^{pq}(\bbL_n^\epsilon) \bbx_{n,l-1}^q \right) -\bbP_n \sigma\left(\sum_{q=1}^{F_{l-1}} \bbh_l^{pq}(\ccalL) f_{l-1}^q\right) \right\|
 \end{align}
 with $\bbx_{n,0}=\bbP_n f$ as the input of the first layer. With a normalized point-wise Lipschitz nonlinearity, we have
  \begin{align}
    \| \bbx_{n,l}^p - \bbP_n f_l^p & \| \leq \left\|  \sum_{q=1}^{F_{l-1}} \bbh_l^{pq}(\bbL_n^\epsilon) \bbx_{n,l-1}^q    - \bbP_n \sum_{q=1}^{F_{l-1}} \bbh_l^{pq}(\ccalL)  f_{l-1}^q\right\|\\
    & \leq \sum_{q=1}^{F_{l-1}} \left\|    \bbh_l^{pq}(\bbL_n^\epsilon) \bbx_{n,l-1}^q    - \bbP_n   \bbh_l^{pq}(\ccalL)  f_{l-1}^q\right\|
 \end{align}
 The difference can be further decomposed as
\begin{align}
   \nonumber   \|    \bbh_l^{pq}(\bbL_n^\epsilon) & \bbx_{n,l-1}^q    - \bbP_n   \bbh_l^{pq}(\ccalL)  f_{l-1}^q \| 
   \\ \nonumber&\leq \|
\bbh_l^{pq}(\bbL_n^\epsilon) \bbx_{n,l-1}^q  - \bbh_l^{pq}(\bbL_n^\epsilon) \bbP_n f_{l-1}^q \\ &\qquad +\bbh_l^{pq}(\bbL_n^\epsilon) \bbP_n f_{l-1}^q  - \bbP_n   \bbh_l^{pq}(\ccalL)  f_{l-1}^q
    \|\\\nonumber
   & \leq \left\|
    \bbh_l^{pq}(\bbL_n^\epsilon) \bbx_{n,l-1}^q  - \bbh_l^{pq}(\bbL_n^\epsilon) \bbP_n f_{l-1}^q
    \right\|
  \\ &\qquad +
    \left\|
    \bbh_l^{pq}(\bbL_n^\epsilon) \bbP_n f_{l-1}^q  - \bbP_n   \bbh_l^{pq}(\ccalL)  f_{l-1}^q
    \right\|
\end{align}
The second term can be bounded with $\| \bbh(\bbL_n^\epsilon)\bbP_n f - \bbP_n\bbh(\ccalL)f \|_{L^2(\bbG_n)}\leq \Delta_{fil,n}$. The first term can be decomposed by Cauchy-Schwartz inequality and non-amplifying of the filter functions as
 \begin{align}
 \left\| \bbx_{n,l}^p - \bbP_n f_l^p \right\| \leq \sum_{q=1}^{F_{l-1}} \Delta_{fil,n}   \| \bbx_{n,l-1}^q\| + \sum_{q=1}^{F_{l-1}} \| \bbx_{l-1}^q - \bbP_n f_{l-1}^{q} \|,
 \end{align}
where $C_{per}$ representing the constant in the error bound of manifold filters in \eqref{eqn:appro_filter}. To solve this recursion, we need to compute the bound for $\|\bbx_l^p\|$. By normalized Lipschitz continuity of $\sigma$ and the fact that $\sigma(0)=0$, we can get
 \begin{align}
 \nonumber &\| \bbx_l^p \|\leq \left\| \sum_{q=1}^{F_{l-1}} \bbh_l^{pq}(\bbL_n^\epsilon) \bbx_{l-1}^{q}  \right\| \leq  \sum_{q=1}^{F_{l-1}}  \left\| \bbh_l^{pq}(\bbL_n^\epsilon)\right\|  \|\bbx_{l-1}^{q}  \| \\
 &\qquad \leq   \sum_{q=1}^{F_{l-1}}   \|\bbx_{l-1}^{q}  \| \leq \prod\limits_{l'=1}^{l-1} F_{l'} \sum_{q=1}^{F_0}\| \bbx^q \|.
 \end{align}
 Insert this conclusion back to solve the recursion, we can get
 \begin{align}
 \left\| \bbx_{n,l}^p - \bbP_n f_l^p \right\| \leq l \Delta_{fil,n}  \left( \prod\limits_{l'=1}^{l-1} F_{l'} \right) \sum_{q=1}^{F_0} \|\bbx^q\|.
 \end{align}
 Replace $l$ with $L$ we can obtain
 \begin{align}
 &\nonumber \|\bm\Phi(\bbH,\bbL_n^\epsilon,\bbP_nf)-\bbP_n \bm\Phi(\bbH,\ccalL, f))\| \\
 &\qquad \qquad \leq \sum_{q=1}^{F_L} \left( L \Delta_{fil,n}  \left( \prod\limits_{l'=1}^{L-1} F_{l'} \right) \sum_{q=1}^{F_0} \|\bbx^q\| \right).
 \end{align}
 With $F_0=F_L=1$ and $F_l=F$ for $1\leq l\leq L-1$, then we have
  \begin{align}
 \|\bm\Phi(\bbH,\bbL_n^\epsilon,\bbP_nf)-\bbP_n \bm\Phi(\bbH,\ccalL, f)) \leq LF^{L-1} \Delta_{fil,n},
 \end{align}
which concludes the proof.

\end{document}